\documentclass{article}
\usepackage[utf8]{inputenc} %
\usepackage[T1]{fontenc}    %
\usepackage{fullpage}

\usepackage[round]{natbib}

\usepackage[margin = 1in]{geometry}

\usepackage{url}            %
\usepackage{booktabs}       %
\usepackage{microtype}      %
\usepackage{wrapfig}

\usepackage{times}
\usepackage{algorithm}
\usepackage{algorithmicx}
\usepackage[noend]{algpseudocode}
\usepackage{adjustbox}
\usepackage{setspace}
\usepackage{dsfont}
\usepackage{amsmath, amssymb}
\usepackage{amsthm, thmtools, thm-restate}
\usepackage{bbm}
\usepackage{graphicx}
\usepackage{latexsym}
\usepackage{mathtools}
\usepackage{multirow}
\usepackage{paralist}
\usepackage{xspace}
\usepackage{enumitem}

\newtheorem{theorem}{Theorem}%
\newtheorem{proposition}[theorem]{Proposition}
\newtheorem{lemma}[theorem]{Lemma}

\newtheorem{definition}{Definition}
\newtheorem{assumption}[definition]{Assumption}

\usepackage{tikz}
\usepackage{pgfplots}
\usetikzlibrary{shapes}
\usetikzlibrary{positioning}
\usetikzlibrary{plotmarks}
\usetikzlibrary{patterns}
\usetikzlibrary{intersections,shapes.arrows}
\usetikzlibrary{pgfplots.fillbetween}
\definecolor{darkpink}{rgb}{0.91, 0.33, 0.5}

\definecolor{puorange}{rgb}{0.80,0.40,0}
\definecolor{bluegray}{rgb}{0.04,0,0.7}
\definecolor{greengray}{rgb}{0.05,0.50,0.15}
\definecolor{darkbrown}{rgb}{0.40,0.2,0.05}
\definecolor{darkcyan}{rgb}{0,0.4,1}
\definecolor{black}{rgb}{0,0,0}
\definecolor{grey}{rgb}{0.93,0.93,0.93}
\definecolor{royalazure}{rgb}{0.0, 0.22, 0.66}

\newcommand{\red}[1]{{\color{purple}#1}}
\newcommand{\green}[1]{{\color{greengray}#1}}
\newcommand{\blue}[1]{{\color{royalazure}#1}}

\usepackage[colorlinks,citecolor=bluegray,linkcolor=darkbrown,urlcolor=blue,breaklinks]{hyperref}

\usepackage[capitalise]{cleveref}

\crefname{section}{Sec.}{Sections}

\crefname{theorem}{Thm.}{Thms.}
\crefname{lemma}{Lem.}{Lems.}
\crefname{corollary}{Cor.}{Cors.}
\crefname{proposition}{Prop.}{Props.}
\crefname{assumption}{Asm.}{Asms.}
\crefname{property}{Propt.}{Propts.}
\crefname{algorithm}{Alg.}{Algs.}
\crefname{appendix}{Appx.}{Appxs.}

\crefname{figure}{Fig.}{Figs.}
\crefname{table}{Tab.}{Tabs.}

\newcommand{\abs}[1]{\left| #1 \right|}
\newcommand{\norm}[1]{\left\lVert #1 \right\rVert}
\newcommand{\smallnorm}[1]{\lVert #1 \rVert}

\newcommand{\zeros}{\operatorname{\mathbf 0}}

\newcommand{\tr}{\operatorname{tr}}

\newcommand{\ip}[1]{{\left\langle #1 \right\rangle}}
\newcommand{\ipsmall}[1]{{\langle #1 \rangle}}

\newcommand{\pow}[1]{^{(#1)}}

\DeclareMathOperator*{\argmin}{arg\,min}

\newcommand{\br}[1]{\ensuremath{\left\{#1\right\}}}

\newcommand{\mc}[1]{\mathcal{#1}}

\newcommand \grad {\nabla}
\newcommand{\R}{\mathbb{R}}
\newcommand{\Ex}{\mathbb{E}}
\newcommand{\E}[2]{\ensuremath{{\mathbb E}_{#1}\left[#2\right]}}
\newcommand{\sbr}[1]{\ensuremath{\left[#1\right]}}
\newcommand{\p}[1]{\ensuremath{\left(#1\right)}}

\newcommand{\Z}{\mathbf{Z}}
\newcommand{\U}{\mathbf{U}}

\newcommand{\x}{\mathbf{x}}
\newcommand{\y}{\mathbf{y}}
\newcommand{\z}{\mathbf{z}}
\renewcommand{\u}{\mathbf{u}}

\renewcommand{\v}{\mathbf{v}}
\newcommand{\s}{\mathbf{s}}

\renewcommand{\r}{\mathbf{r}}
\newcommand{\e}{\mathbf{e}}
\newcommand{\btheta}{\boldsymbol{\theta}}
\newcommand{\w}{\mathbf{w}}
\newcommand{\W}{\mathbf{W}}

\newcommand{\A}{\mathbf{A}}
\newcommand{\B}{\mathbf{B}}
\newcommand{\g}{\mathbf{g}}

\newcommand{\logdet}[1]{\log\left|\operatorname{det}\left(#1\right)\right|}
\newcommand{\laplace}{\operatorname{Laplace}}
\newcommand{\etau}{\eta_{\mathrm{u}}}
\newcommand{\etap}{\eta_{\mathrm{p}}}
\newcommand{\etaa}{\eta_{\mathrm{a}}}

\newcommand{\amari}{\operatorname{AD}}

\newcommand{\Lunsup}{L_{\mathrm{u}}}
\newcommand{\Lsup}{L_{\mathrm{s}}}

\newcommand{\const}{\mathrm{const}}
\newcommand{\Fro}{\mathrm{F}}
\newcommand{\algoname}{MultiICA\xspace}

\title{Supervised Stochastic Gradient Algorithms\\ for Multi-Trial Source Separation}
\author{Ronak Mehta$^{1}$ \qquad Mateus Piovezan Otto$^{1}$ \qquad Noah Stanis$^{2}$ \vspace{0.1cm}\\
Azadeh Yazdan-Shahmorad$^{2}$ \qquad Zaid Harchaoui$^1$ \vspace{0.3cm} \\
{
\small $^1$Department of Statistics, University of Washington
\qquad
$^2$Department of Bioengineering, University of Washington
}
}
\date{\today}

\begin{document}

\maketitle

\begin{abstract}
We develop a stochastic algorithm for independent component analysis that incorporates multi-trial supervision, which is available in many scientific contexts. The method blends a proximal gradient-type algorithm in the space of invertible matrices with joint learning of a prediction model through backpropagation. We illustrate the proposed algorithm on synthetic and real data experiments. In particular, owing to the additional supervision, we observe an increased success rate of the non-convex optimization and the improved interpretability of the independent components.

\end{abstract}

\section{Introduction}\label{sec:intro}

A fundamental inverse problem in statistical/machine learning, signal processing, and time series analysis is independent component analysis, or ICA \citep{hyvarinen2001independent, cichocki2002adaptive, comon1010handbook}. Given a signal $\z \in \R^{C \times T}$ with $C$ components/channels and $T$ samples, the scientist pursues a \emph{source} signal $\s \in \R^{C \times T}$ and an invertible \emph{mixing matrix} $\A \in \R^{C \times C}$ such that 1) the rows of $\s$ are realizations of independent stochastic processes, and 2) the equality
\begin{align}
    \z = \A\s.
    \label{eq:ica}
\end{align}
is satisfied. Equivalently, we seek an invertible \emph{unmixing matrix} $\W \in \R^{C \times C}$ which recovers $\A^{-1}$ up to permutations and non-zero scalings of the rows (which do not affect the independence criterion). Applications include identifying individual voices/tracks in audio or disentangling correlated electrical activity from the brain. Often viewed as a data exploration/preprocessing technique, the independent sources are interpreted as underlying ``drivers'' of the signal, whereas the mixing matrix summarizes its connectivity/correlation structure via a simple linear transformation. Despite its widespread use in the natural and data sciences, the fully unsupervised setting of ICA limits its interpretability (ostensibly its main benefit over complex, nonlinear separation functions), as the contribution of each source may not be easily understood in a scientifically meaningful way. As an example of such an understanding, one may consider the identities of the voices in audio or the biological/behavioral function of the brain signals mentioned above. Moreover, from a technical viewpoint, ICA algorithms often hinge upon solving non-convex optimization problems over the space of invertible matrices, for which even state-of-the-art algorithms may fail consistently. Given the motivation from both a methods and an applications perspective, we develop in this paper a stochastic algorithm for solving the Infomax variant \citep{bell1995aninformation, lee1999independent, amari1999natural} of ICA that incorporates \emph{multi-trial supervision} in a flexible, model-agnostic fashion.

Precisely, consider a dataset of $N$ observations $(\z_1, \y_1), \ldots, (\z_N, \y_N)$, where each $\z_i \in \R^{C \times T}$ is a multivariate signal and each $\y_i = (y_{i, 1}, \ldots, y_{i, M})$ is a collection of discrete or continuous labels. We seek an unmixing matrix $\W$ that can be applied to \emph{all} signals, using the given labels as additional guidance to uncover the sources. This formalization is heavily motivated by the neuroscience application mentioned above, as in this setting, signals are often collected via multiple trials with possibly different conditions, interventions, and behaviors of the individual being measured. The aforementioned interpretability goal may be realized if (some number of) independent sources have a direct correspondence to each supervision label. Accordingly, assuming that $M \leq C$, we consider optimization problems of the form
\begin{align}
    \min_{\W, \btheta} \ \frac{1}{N}\sum_{i=1}^N \sbr{\ell_0(\W, \z_i) + \lambda \sum_{m=1}^M  \ell_m(\W_{m \cdot}^\top \z_i, y_{i, m}, \btheta_m)},
    \label{eq:infomax}
\end{align}
where $\ell_0$ is the unsupervised loss that promotes independence between the sources, $\ell_1, \ldots, \ell_M$ denote supervised losses that promote dependence between individual sources and labels, $\W_{m \cdot}$ denotes the $m$-th row of $\W$, $\btheta = (\btheta_1, \ldots, \btheta_M)$ denotes the parameters of predictive models (such as linear transformations or neural networks), and finally, $\lambda \geq 0$ is a balancing hyperparameter. In general, the objective may be non-convex in both $\W$ and $\btheta$, and each $\ell_m$ is only assumed to be differentiable in its first and third arguments.

\paragraph{Contributions}
We propose a stochastic optimization algorithm for~\eqref{eq:infomax} which combines a proximal gradient-type update in $\W$ (effectively handling the invertibility constraint) and a generic gradient-based update scheme of the user's choice for $\btheta$. In particular, one may use the Adam/AdamW class of algorithms \citep{loshchilov2018decoupled}, popularly applied in neural network training, to update $\btheta$.\footnote{In our experiments, we find that adaptive updates improve performance even when the objective is convex in $\btheta$.} We prove a monotonicity guarantee and convergence to a stationary point of the objective. On experiments with synthetic data, we demonstrate the ability of supervision to decrease the failure rate of non-convex optimization trajectories, as illustrated in \Cref{fig:nonconvex}. On real neural data collected from non-human primates, we demonstrate the ability of the individual sources to retain scientific information related to the experimental protocol and the animal's behavior during each trial.
We derive the algorithm in \Cref{sec:methods}, provide experiments on synthetic and real neural data in \Cref{sec:experiments}, and give concluding remarks in \Cref{sec:conclusion}.

\paragraph{Related Work}
Both direct gradient and natural/Riemannian gradient approaches have been employed for solving ICA under the Infomax principle, often facing convergence issues due to the likelihood-based loss $\ell_0$ including a log-determinant component \citep{martinez2017caveats}. While ICA algorithms that use a weighted sum of unsupervised and supervised components in the objective have been explored previously, they either apply gradient-based methods to $\W$ that inherit the aforementioned convergence issues or are model-specific with respect to $\btheta$, i.e., place strong assumptions on the supervised losses $\ell_1, \ldots, \ell_M$ \citep{chen2002supervised, kotani2004supervised, takabatake2007feature}. Other approaches for learning $\W$ exist, such as fixed-point/matrix decomposition schemes, which also may not adapt to arbitrary supervision models \citep{zou2022asupervised, su2024disentangling}. We maintain generic gradient-based optimization via backpropagation for $\btheta$, which is undisputably effective for supervised learning problems. For the $\W$ update, we are particularly inspired by the majorization-minimization class of approaches \citep{ono2011stable, ablin2019stochastic, schibler2020fast, brendel2020spatially, brendel2021accelerating, ikeshita2021block, ikeshita2022iss2}, of which \citet{ablin2019stochastic} minimizes a per-iteration objective row-by-row. We generalize their per-iteration objective both by incorporating a proximity term to stabilize the $\W$ trajectory and a linearized supervision term, which can be computed via backpropagation. A proximal approach for the related problem of Gaussian independent vector analysis has been explored \citep{cosserat2023aproximal}, but not for ICA with a generic likelihood and supervision model.
On the applied side, the usage of independent component/vector analysis on neural and cognition data is thoroughly established \citep{lehmann2022multi, moraes2023applying, yang2023constrained, fouladivanda2023joint, heurtebise2023multiview, keding2024effect, laport2024reproducibility, vu2024arobust, gjolbye2024speed, hu2025reference}. Although we focus on interpreting linearly mixed sources in this paper, nonlinear deep learning-based variants of ICA have also gained interest recently \citep{nguyen2021deep, li2022deep, narisawa2021independent, hermann2022large, romano2023using}.

\begin{figure}
    \centering
    \includegraphics[width=0.7\linewidth]{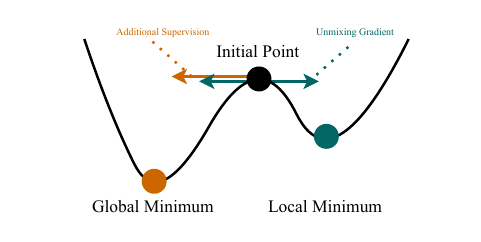}
    \caption{{\bf Non-Convex Optimization Landscape.} Gradient information from the unmixing (unsupervised) component of the objective may not easily differentiate between local and global minima, motivating supervision from auxiliary targets. }
    \label{fig:nonconvex}
\end{figure}

\section{Methods}\label{sec:methods}
\paragraph{Variational Objective}
To specify the method, we first write a variational form of the objective~\eqref{eq:infomax}, which introduces a third set of variables we denote as $\U \in \R^{N \times C \times T}$. The algorithm jointly minimizes over $(\W, \btheta, \U)$, and is fully specified by first providing full batch updates for each variable, and then describing stochastic estimates thereof. To proceed, we require the following mild assumption on our unmixing loss function $\ell_0$.
\begin{assumption}[Super-Gaussian Likelihood]\label{asm:supergaussian}
    For any invertible $\W \in \R^{C \times C}$ and $\z \in \R^{C \times T}$, it holds that 
    \begin{align*}
        \ell_0(\W, \z) = L(\W) + \frac{1}{T}\sum_{c=1}^C \sum_{t=1}^T g([\W \z]_{c, t}),
    \end{align*}
    where $x \mapsto e^{-g(x)}$ is integrable, $x \mapsto g(\sqrt{x})$ is increasing and concave on $(0, \infty)$, and $L(\W) :=  -\logdet{\W}$.
\end{assumption}
Under \Cref{asm:supergaussian}, a constant can be added to $g$ so that, without loss of generality, $e^{-g(\cdot)}$ is a probability density function for a \emph{super-Gaussian} random variable by definition. This includes the Laplace density $g(x) \sim \abs{x}$ or the Huber density $g(x) \sim h(x)$, where $h(x) = \frac{1}{2}x^2$ for $\abs{x} \leq 1$ and $h(x) = \abs{x} - 1/2$ otherwise. Such a function would naturally arise when using the density transformation formula to express the likelihood of each entry $z_{c, t}$ of $\z$ using the likelihood of each entry of the candidate source signal $\W \z$.
In turn, \citet[Theorem 1]{palmer2005variational} grants the variational form
\begin{align}
    \sum_{c=1}^C \sum_{t=1}^T g([\W \z]_{c, t}) = \min_{\u \geq \zeros_{C \times T}}  \sum_{c=1}^C \sum_{t=1}^T \frac{1}{2} u_{c, t} [\W \z]_{c, t}^2  + f(u_{c, t}), \label{eq:supergaussian}
\end{align}
where $\zeros_{C\times T}$ denotes the matrix of zeros in $\R^{C\times T}$, and $f: [0, \infty) \rightarrow \R$ is a convex function whose particular form is unimportant for our purposes. Combining~\eqref{eq:infomax},~\eqref{eq:supergaussian}, and \Cref{asm:supergaussian}, we derive the equivalent problem to~\eqref{eq:infomax} of minimizing
\begin{align}
    F(\W, \btheta, \U) := L(\W) + \frac{1}{N}\sum_{i=1}^N \sbr{\sum_{c=1}^C \sum_{t=1}^T \frac{1}{2} U_{i, c, t} [\W \z_i]_{c, t}^2 + f(U_{i, c, t}) + \lambda \sum_{m=1}^M  \ell_m(\W_{m \cdot}^\top \z_i, y_{i, m}, \btheta_m)}, \label{eq:full_objective}
\end{align}
over $\W$, $\btheta$, and $\U$, which ranges over the non-negative tensors within $\R^{N \times C \times T}$.
We proceed to describe the updates for the three sets of variables, defining the sequence $(\U\pow{k}, \btheta\pow{k}, \W\pow{k})_{k \geq 0}$. They are performed in a block coordinate-wise manner, ordered by the auxiliary variables $\U\pow{k}$ first, the model parameters $\btheta\pow{k}$ second, and the unmixing matrix $\W\pow{k}$ last. The majority of the complexity lies in the $\W\pow{k}$ update, whereas the other variables can be updated using exact minimizations or gradient-based updates.

\paragraph{Auxiliary Variables and Model Parameters Update}

For the auxiliary variables, we simply perform the exact minimization
\begin{align}
    \U\pow{k} \gets \argmin_{\U \geq \zeros_{N \times C \times T}} F(\W\pow{k-1}, \btheta\pow{k-1}, \U).\label{eq:aux_update}
\end{align}
Differentiating both sides of~\eqref{eq:supergaussian} with respect to each $[\W \z]_{c, t}$ yields that the optimum $\U\pow{k}$ is achieved by setting $U_{i, c, t}\pow{k} = g'([\W \z_i]_{c, t}) / [\W \z_i]_{c, t}$, which can be done in a vectorized manner. As mentioned in \Cref{sec:intro}, we apply a generic gradient-based update scheme for the model parameter $\btheta\pow{k}$. That is, for each $m = 1, \ldots, M$, we have
\begin{align*}
    \btheta\pow{k}_m \gets \operatorname{GradientBasedUpdate}(\btheta\pow{k-1}_m, \etap, \g\pow{k-1})
\end{align*}
for model parameter learning rate $\etap > 0$ and gradient formula
\begin{align}
    \g_m\pow{k} := \grad_{\btheta_m}F(\W\pow{k-1}, \btheta, \U\pow{k})\big{|}_{\btheta_m = \btheta\pow{k-1}_m} = \frac{1}{N}\sum_{i=1}^N \grad_{\btheta_m} \ell_m((\W\pow{k-1}_{m\cdot})^\top \z_i, y_{i, m}, \btheta_m)\big{|}_{\btheta_m = \btheta\pow{k-1}_m}.\label{eq:gvec}
\end{align}
For example, $\btheta\pow{k}_m = (1-\etap \mu)\btheta\pow{k}_m - \etap \g\pow{k}_m$ represents gradient descent with weight decay parameter $\mu > 0$. Algorithms such as AdamW \citep{loshchilov2018decoupled} can also be employed, with update
\begin{align*}
    \btheta\pow{k}_m = (1- \etap \mu)\btheta\pow{k-1}_m - \etap \lambda \frac{\mathbf{m}\pow{k}_m}{(\mathbf{v}\pow{k}_m)^{1/2} + \epsilon}
\end{align*}
for momentum-based stochastic estimate $\mathbf{m}\pow{k}_m$ of~\eqref{eq:gvec}, variance pre-conditioner $\mathbf{v}\pow{k}_m$, and tolerance parameter $\epsilon > 0$. The division and square root operations are interpreted element-wise. We employ this update in the experiments shown in \Cref{sec:experiments}.

\paragraph{Unmixing Matrix Update}

We consider a cyclic coordinate-wise procedure in which $\W\pow{k}$ will be defined by a sequence of intermediate values $\W\pow{k, 0}, \ldots, \W\pow{k, C}$, such that $\W\pow{k, 0} = \W\pow{k-1}$, $ \W\pow{k} = \W\pow{k, C}$, and $\W\pow{k, c}$ will differ from $\W\pow{k, c-1}$ by updating the $c$-th row based on minimizing an approximation of our original objective. In particular,
\begin{align}
    \W\pow{k, c} \gets \argmin_{\substack{\W_{j \cdot} = \W_{j \cdot}\pow{k, c-1}, j \neq c}} F\pow{k, c-1}(\W, \btheta\pow{k}, \U\pow{k}),\label{eq:W_update0}
\end{align}
where $F\pow{k, c-1}$ approximates~\eqref{eq:full_objective} in the region close to $\W\pow{k, c-1}$. The definition of $F\pow{k, c-1}$ will enforce that $\W\pow{k, c}$ remains invertible. By describing this per-iteration objective and the means to optimize it, we fully specify the update of the unmixing matrix and the overall algorithm. Three technical ingredients form the basis for this objective function: a linearization of the supervised component, a proximity term to promote closeness to $\W\pow{k, c-1}$, and a reparametrization that allows for the minimization to be solved in closed form.

To describe the first two components, we first write the original objective in a simplified manner using
\begin{align}
    F(\W, \btheta\pow{k}, \U\pow{k}) = \underbrace{L(\W) + \frac{1}{2}\sum_{c=1}^C \W_{c\cdot}^\top \A_c\pow{k} \W_{c\cdot}}_{\text{unsupervised component}} + \underbrace{\frac{\lambda}{N} \sum_{i=1}^N  \sum_{m=1}^M  \ell_m(\W_{m\cdot}^\top \z_i, y_{i, m}, \btheta_m\pow{k})}_{\text{supervised component}} + \const(\W),\label{eq:full_objective2}
\end{align}
where $\const(\cdot)$ is independent of its input and $\A_1\pow{k}, \dots, \A_C\pow{k}$ are $C$ square matrices defined by
\begin{align}
    \A_c\pow{k} := \frac{1}{NT}\sum_{i=1}^N \sum_{t=1}^T U_{c, i, t}\pow{k} \z_{i, t} \z_{i, t}^\top. \label{eq:Amat}
\end{align}
The supervised portion of~\eqref{eq:full_objective2} can be linearized by its matrix derivative. Letting $\ip{\A, \B} := \tr(\A^\top \B)$ for $\A, \B \in \R^{C \times C}$, we specify~\eqref{eq:W_update0} by defining
\begin{align}
    F\pow{k, c-1}(\W, \btheta\pow{k}, \U\pow{k}) := L(\W) + \frac{1}{2}\sum_{c=1}^C \W_{c\cdot}^\top \A_c\pow{k} \W_{c\cdot} + \lambda\ipsmall{\B_c\pow{k}, \W} + \frac{1}{2\etau}\norm{\W - \W\pow{k, c-1}}_{\Fro}^2,\label{eq:full_objective3}
\end{align}
for the square matrix
\begin{align}
    \B_c\pow{k} := \frac{1}{N} \sum_{i=1}^N  
    \begin{bmatrix}
        \grad_{\s} \ell_1(\s, y_{i, 1}, \btheta_1\pow{k})\big{|}_{\s = (\W_{1\cdot}\pow{k, c-1})^\top \z_i}\\
        \vdots\\
        \grad_{\s} \ell_M(\s, y_{i, M}, \btheta_M\pow{k})\big{|}_{\s = (\W_{M\cdot}\pow{k, c-1})^\top \z_i}\\
        \zeros_{(C - M) \times T}
    \end{bmatrix}
    \z_i^\top,  \label{eq:Bmat}
\end{align}
unmixing learning rate $\etau > 0$ and Frobenius norm $\norm{\cdot}_{\Fro}$. To motivate the third and final ingredient, notice the difficulty of optimizing the non-smooth, non-convex objective~\eqref{eq:full_objective3} resulting from the log-determinant term $L(\W)$. To handle this, we follow a similar reparametrization trick to \citet[Lemma 3]{ablin2019stochastic}. Let $\e_{l:m} \in \R^{C \times (m - l)}$ denote the matrix containing the $l$-th through $m$-th standard basis vectors along its columns, and notice that due to invertibility of $\W\pow{k, c-1}$, it holds that
\begin{align}
    \W\pow{k, c} = (\e_{1:c-1}, \r_c, \e_{c+1:C})^\top\W\pow{k, c-1}\label{eq:reparametrization}
\end{align}
for a vector satisfying $(\W\pow{k, c}_{c \cdot})^\top = \r_c^\top \W\pow{k, c-1}$, or equivalently, $\r_c = ((\W\pow{k, c-1})^\top)^{-1} \W\pow{k, c}_{c \cdot} \in \R^C$. By substituting $\W$ in~\eqref{eq:full_objective3} with the right-hand side of~\eqref{eq:reparametrization}, we can solve the per-iteration problem~\eqref{eq:W_update0} by optimizing for $\r_c$ directly over $\R^C$. To hint as to why this would be useful, observe that
\begin{align*}
    L(\W) = L((\e_{1:c-1}, \r_c, \e_{c+1:C})^\top\W\pow{k, c-1}) = \log \abs{r_{c, c}} + L(\W\pow{k, c-1}),
\end{align*}
where $r_{c, c} \neq 0$ for any feasible $\W$, reducing this term to a univariate function in the decision variables.
The exact solution to~\eqref{eq:W_update0} is given by \Cref{prop:update} below, where we also define the standard basis vector $\e_c := \e_{c:c}$.
\begin{restatable}{proposition}{update}\label{prop:update}
    Assume that $\A_c\pow{k}$ is positive definite for all $c \in [C]$. Each $\r_c$ from~\eqref{eq:reparametrization} is given by
    \begin{align*}
        \r_c := \mathbf{K}^{-1} \p{\frac{1}{r_{c, c}} \e_c + \mathbf{b}},
    \end{align*}
    for $\mathbf{K} := \W\pow{k, c-1} (\A_c\pow{k} + \etau^{-1}\mathbf{I})(\W\pow{k, c-1})^\top$, $ \mathbf{b} := \W\pow{k, c-1} (\etau^{-1}\W_{c\cdot}\pow{k, c-1} -  (\B_c\pow{k})_{c \cdot})$, and $r_{c, c} := \sqrt{\mathbf{K}^{-1}_{cc} + \tfrac{1}{4}(\mathbf{K}^{-1} \mathbf{b}})_c^2 +\tfrac{1}{2}(\mathbf{K}^{-1}\mathbf{b})_c$.
\end{restatable}
As an important simplification, notice that the update given by \Cref{prop:update} only depends on $\B_c\pow{k}$ through its $c$-th row, which, inspecting ~\eqref{eq:Bmat}, in turn depends on $\W_{c\cdot}\pow{k, c-1}$. However, it always holds that $\W_{c\cdot}\pow{k, c-1} = \W_{c\cdot}\pow{k, 0} = \W_{c\cdot}\pow{k - 1}$, because the $c$-th row of $\W\pow{k-1}$ has not yet been updated. Thus, one can compute the matrix $\B_1\pow{k}$ only once and use the same one for all updates, which is reflected in \Cref{algo:multiica}.

\paragraph{Convergence Analysis}
To justify the three updates above, we also derive the conditions under which they yield a monotonically non-increasing sequence of objective values, commonly sought in structured non-convex problems such as ICA and mixture modeling. We then convert this guarantee into an asymptotic convergence analysis of the sequence under assumptions on the optimization trajectory. For concreteness, we consider the gradient descent update $\btheta\pow{k}_m = (1-\etap \mu)\btheta\pow{k}_m - \etap \g\pow{k}_m$, which applies to the $\ell_2$-regularized objective $F_\mu(\W, \btheta, \U) = F(\W, \btheta, \U) + \frac{\mu}{2}\norm{\btheta}_2^2$. Monotonicity can be achieved when the supervised objective is \emph{smooth} with respect to the source and parameter. 
\begin{assumption}\label{asm:smoothness}
    Assume the following for any fixed target $\y = (y_1, \ldots, y_m)$, parameter $\btheta = (\btheta_1, \ldots, \btheta_M)$, and source component $\s \in \R^{T}$.
    The function $\s \mapsto \grad_{\s} \ell_m(\s, y_{m}, \btheta_m)$ is $L_{m}$-Lipschitz continuous w.r.t.~$\norm{\cdot}_{2}$, the function $\btheta_m \mapsto \grad_{\btheta} \ell_m(\s, y_{m}, \btheta_m)$ is $L_{\btheta}$ w.r.t.~$\norm{\cdot}_2$, and  $\s \mapsto \grad_{\btheta} \ell_m(\s, y_{m}, \btheta_m)$ is $\bar{L}$-Lipschitz continuous w.r.t.~$\norm{\cdot}_{2}$.
\end{assumption}
This is the only assumption required for the descent guarantee \Cref{lem:monotonicity}, which is proven in \Cref{sec:a:descent}.
\begin{restatable}{lemma}{monotonicity}\label{lem:monotonicity}
    Under \Cref{asm:supergaussian}, \Cref{asm:smoothness}, and the conditions
    \begin{align*}
        \etau \leq \sbr{2\lambda \p{\textstyle\tfrac{1}{N}\sum_{i=1}^N \norm{\z_i}_{2, 2}^2} \sqrt{\textstyle\sum_{m=1}^M L_m^2}}^{-1} \text{ and } \etap \leq \p{L_{\btheta} + \mu}^{-1},
    \end{align*}
    we have that for all $k \geq 1$, the inequality
    \begin{align}
        F_\mu(\W\pow{k}, \btheta\pow{k}, \U\pow{k}) \leq F_\mu(\W\pow{k-1}, \btheta\pow{k-1}, \U\pow{k-1}).\label{eq:descent}
    \end{align}
    holds. Consequently, if $F$ is bounded from below, then $F_\mu(\W\pow{k}, \btheta\pow{k}, \U\pow{k})$ converges to a finite limit $F^\star$ as $k \rightarrow \infty$. 
\end{restatable}
The result is proven by combining descent inequalities for each update, including each inner loop iteration of updates to the rows of the unmixing matrix. Practically, we see that $\etau$ can be scaled using the hyperparameter $\lambda$ and the average spectral norm $\norm{\cdot}_{2, 2}$ of the input signals, which can be normalized during preprocessing. For convergence to a stationary point, we require a number of additional assumptions and a slight modification of the algorithm, which introduces a proximity term to the $\U\pow{k}$ update. First, we define $\norm{\U}_{\Fro}^2 := \sum_{i, c, t} U_{i, c, t}^2$ for $\U \in \R^{N \times C \times T}$, and we change~\eqref{eq:aux_update} to read as
\begin{align}
    \U\pow{k} \gets \argmin_{\U \geq \zeros^{N \times C \times T}} \br{F(\W\pow{k-1}, \btheta\pow{k-1}, \U) + \frac{1}{2\etaa} \sum_{i, c, t} (U_{i, c, t} - U_{i, c, t}\pow{k-1})^2},\label{eq:aux_update2}
\end{align}
which requires knowledge of the function $f$ in~\eqref{eq:full_objective} in order to implement. We find in experimentation that the proximal variant of the update above performs indistinguishably from the exact minimization, which can be implemented without ever deriving $f$ for a particular likelihood. However, this variant leads to stronger theoretical guarantees. To gracefully handle the requirement that $\U \geq \zeros_{N \times C \times T}$ when minimizing~\eqref{eq:full_objective}, we define the convex indicator function $\iota_{+}: \R^{N \times C \times T} \rightarrow \br{0, + \infty}$ such that $\iota_{+}(\U) = 0$ when $\U \geq \zeros_{N \times C \times T}$ and $\iota_{+}(\U) = +\infty$ otherwise, and denote
\begin{align}
    \Psi(\U, \btheta, \W) &= F_\mu(\W, \btheta, \U) + \iota_+(\U).\label{eq:full_objective_frechet}
\end{align}
Minimizing $\Psi$ over $(\U, \btheta, \W)$ is a non-smooth, non-convex optimization problem. To define stationarity formally, we require the concept of a \emph{Fr\'{e}chet subdifferential}; we leave the technical details to \Cref{sec:a:stationarity} and invite the reader to interpret stationarity in the usual ``zero gradient'' sense in the result below.
\begin{proposition}\label{prop:stationarity}
    Assume the conditions of \Cref{lem:monotonicity} and additionally, \Cref{asm:bdd1,asm:bdd2,asm:bdd3} from \Cref{sec:a:stationarity}. The sequence of iterates $(\W\pow{k}, \btheta\pow{k}, \U\pow{k})_{k \geq 0}$ produced by \Cref{algo:multiica} using the update~\eqref{eq:aux_update2} converges to a stationary point of the objective~\eqref{eq:full_objective_frechet}.
\end{proposition}
The proof is given in \Cref{sec:a:stationarity} and relies on two broad steps. The first is to use an improvement of the inequality~\eqref{eq:descent} to conclude that $(\W\pow{k-1} - \W\pow{k}, \btheta\pow{k-1} - \btheta\pow{k}, \U\pow{k-1} - \U\pow{k})$ converges to zero. Then, we compute a particular sequence of Fr\'{e}chet subgradients whose norms can be upper bounded by norms of the gaps between successive iterates. 
The second step is non-trivial, as the \Cref{algo:multiica} combines standard proximal, proximal gradient, and cyclic coordinate-wise updates which each have different optimality conditions.

\paragraph{Full Batch and Stochastic Variants}
We describe the computational complexity per iteration of both the full batch \Cref{algo:multiica} and a natural stochastic/incremental variant, which leverages the finite-sum structure of the various quantities at play. The three memory bottlenecks arise from the storage of the auxiliary variables $\U\pow{k}$, the model parameter $\btheta \in \R^d$, and the matrices $\A\pow{k}_1, \ldots, \A\pow{k}_C \in \R^{C \times C}$. Observing that $\A\pow{k}_{c-1}$ can be discarded on iterate $c$ of the inner loop, the total space complexity of the full batch algorithm is $O(NCT + C^2 + d)$. For the time complexity, assuming that the gradient-based $\btheta\pow{k}$ update can occur in $O(Nd)$ time, the remaining bottleneck is the computation of $\{(\A\pow{k}_c, \B_c\pow{k})\}_{c=1}^C$. Because $M \leq C$, this yields a time complexity of $O(N(d + C^{3} T))$.

While the space complexity simply reflects the memory needed to store the decision variables and data, the time complexity can be improved using sample average approximations of the matrices $(\A\pow{k}_1, \B\pow{k}_1), \ldots,(\A\pow{k}_C, \B\pow{k}_C)$. In particular, consider indices $\mc{I} = (i_1, \ldots, i_n)$ drawn uniformly without replacement from $[N]$ and $\mc{T} = (t_1, \ldots, t_\tau)$ drawn uniformly without replacement from $[T]$. Letting $\Ex_{k-1}$ denote the conditional expectation over the randomness in these indices given $(\U\pow{k-1}, \btheta\pow{k-1}, \W\pow{k-1})$, we have the three identities 
\begin{align*}
    \g_m\pow{k} &= \E{k-1}{\frac{1}{n}\sum_{i \in \mc{I}} \grad_{\btheta_m} \ell_m((\W\pow{k-1}_{m\cdot})^\top \z_i, y_{i, m}, \btheta_m)\big{|}_{\btheta_m = \btheta\pow{k-1}_m}}, &\text{ for } m = 1,\ldots, M\\
    \A_c\pow{k} &= \E{k-1}{\frac{1}{n\tau}\sum_{i \in \mc{I}} \sum_{t \in \mc{T}} U_{c, i, t}\pow{k} \z_{i, t} \z_{i, t}^\top \Big{|} \U\pow{k}}, &\text{ for } c = 1,\ldots, C\\
    \B_c\pow{k} &= \E{k-1}{\frac{1}{n} \sum_{i \in \mc{I}} \sum_{m=1}^M  \frac{T}{\tau} \sum_{t \in \mc{T}} \p{\grad_{\s} \ell_m(\s, y_{i, m}, \btheta_m\pow{k})\big{|}_{\s = (\W_{m\cdot}\pow{k, c-1})^\top \z_i}}_{\cdot t}\e_m\z_{i, t}^\top \Big{|} \btheta\pow{k} }.
\end{align*}
To derive a stochastic version of the algorithm, we not only use the integrands above as sample average approximations of their expectations,
but we also edit line~\ref{line:u_update} to only compute $U\pow{k}_{i, c, t}$ for $i \in \mc{I}$ and $t \in \mc{T}$, setting $U\pow{k}_{i, c, t} \gets U\pow{k-1}_{i, c, t}$ otherwise. The time complexity of the resulting algorithm is given by
\begin{align}
    O\p{n(d + MCT + \tau C^3)}.\label{eq:time_complexity}
\end{align}
The first term comes from the estimation of $\g\pow{k}$, the second term comes from the estimation of $\B_c\pow{k}$, and the third term results from the estimation of each $\A\pow{k}_c$. Importantly, even though only $\tau$ timesteps are used for $\B_c\pow{k}$, the full gradient with respect to the $\R^{C \times T}$-valued input is needed before subsetting. After this, the remaining computation needed to compute the $C \times C$ matrix is $O(nMC^2 \tau)$, which is absorbed into the final term of~\eqref{eq:time_complexity} because $M \leq C$.
We evaluate the method in a variety of synthetic and real data settings in the next section.

\begin{algorithm}[t]
\setstretch{1.2}
   \caption{Multi-Trial Supervised Independent Component Analysis (\algoname)}
   \label{algo:multiica}
    \begin{algorithmic}[1]
       \State {\bfseries Input:} Number of iterations $K$, learning rates $(\etap, \etau)$, balancing constant $\lambda$, loss $\ell$, and function $g$ (see \Cref{asm:supergaussian}).
       \State Randomly initialize $\W\pow{0}$ and $\btheta\pow{0}$.
       \For {$k=1$ {\bfseries to} $K$}
           \For {$m=1$ {\bfseries to} $M$}
                \State $\btheta\pow{k}_m = \operatorname{GradientBasedUpdate}(\btheta\pow{k-1}_m, \etap, \g\pow{k-1})$ (see~\eqref{eq:gvec}).\label{line:supervised_update}
           \EndFor
           \State $U\pow{k}_{i, c, t} := g'(\ipsmall{\W_{c\cdot}\pow{k-1}, \z_{i, t}}) / \ipsmall{\W_{c\cdot}\pow{k-1}, \z_{i, t}}$ for all $(i, c, t)$.\label{line:u_update}
           \State Compute $\B_1\pow{k}$ from~\eqref{eq:Bmat}.
           \For {$c=1$ {\bfseries to} $C$}
                \State Compute~$\A\pow{k}_c$ from~\eqref{eq:Amat}.
                \State Compute $\r_c$ via \Cref{prop:update} (setting $\B_c\pow{k} \gets \B_1\pow{k}$).\label{line:update}
                \State Compute $\W\pow{k, c}$ via~\eqref{eq:reparametrization}.
           \EndFor
           \State $\W\pow{k} = \W\pow{k, C}$.
       \EndFor
       \Return Unmixing matrix $\W\pow{K}$ and parameters $\btheta\pow{K}$.
    \end{algorithmic}
\end{algorithm}

\section{Experimental Results}\label{sec:experiments}
\begin{figure}[t]
    \centering
    \includegraphics[width=0.6\linewidth]{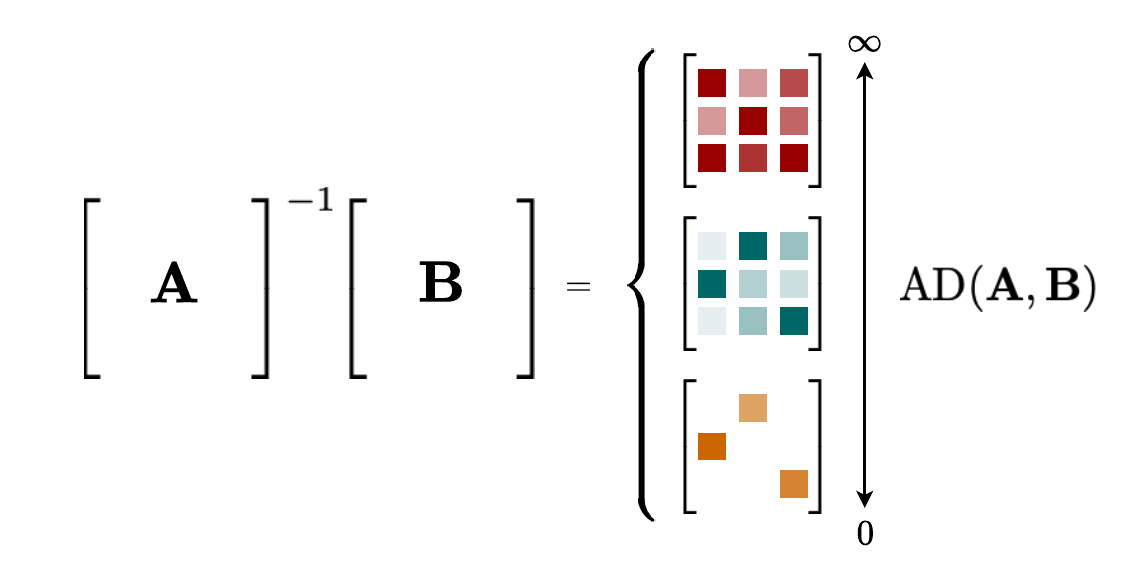}
    \caption{{\bf Amari Distance.} Illustration of various matrix products $\A^{-1}\B$ ranked by the Amari distance $\amari(\A, \B)$. Intensities of the colored squares represent absolute value/magnitude, so that empty squares indicate zero entries.}
    \label{fig:amari}
\end{figure}

\begin{figure*}[t]
    \centering
    \includegraphics[width=\linewidth]{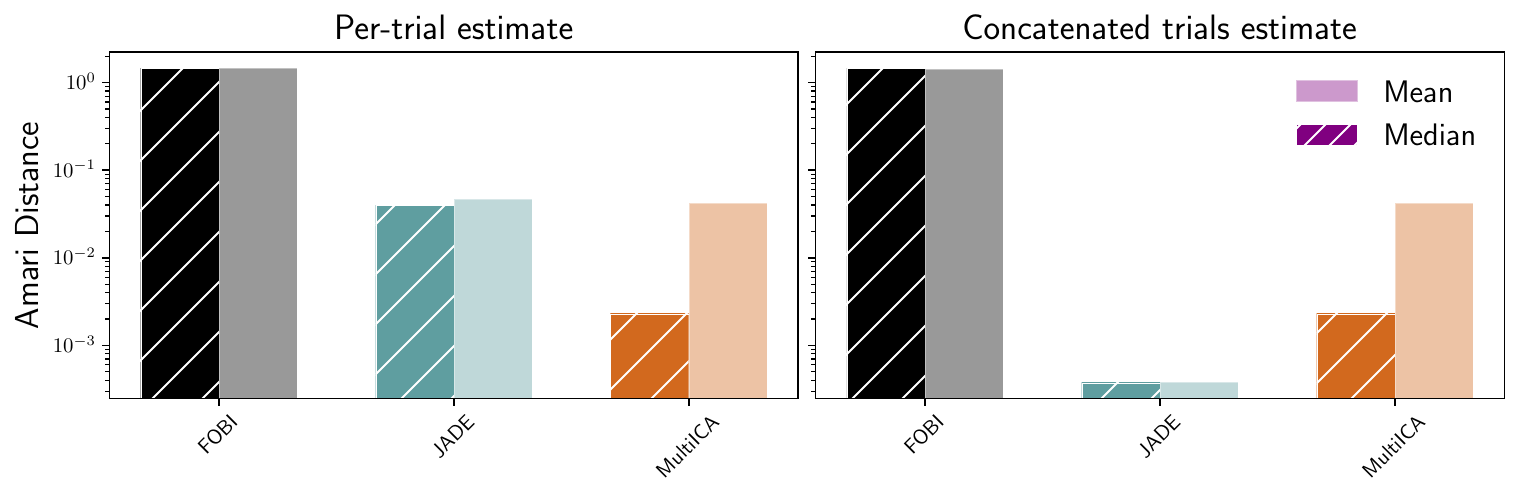}
    \caption{{\bf Results: Multiple Trials.} The left panel shows the mean and median Amari distances over $N = 80$ trials for FOBI and JADE against a run of \algoname. The right panel shows the mean and median Amari distances for FOBI and JADE applied to the concatenation of all $N = 80$ trials along the time axis compared to a run of \algoname. \algoname was run for $K = 10^4$ iterations. The mean and median Amari distance are taken over trials and $50$ seeds in the left panel, and over $50$ seeds in the right panel.}
    \label{fig:multi}
\end{figure*}

In our experiments, we aim to determine the benefits of incorporating both multiple trials and supervision into the learning process. 
Via both simulated and real data experiments, we also provide practical recommendations on hyperparameter tuning for \Cref{algo:multiica}. 
As for quantitative measurement, we recall from \Cref{sec:intro} that we are interested in recovering $\A$ via $\W^{-1}$ down to permutations and scalings of the rows. Thus, we measure success via the \emph{Amari distance} criterion
\begin{align}
    \amari(\W^{-1}, \A) = \sum_{j=1}^C \sbr{\p{\sum_{c=1}^C \frac{R_{jc}}{\max_{c'} R_{jc'}}-1} + \p{\sum_{c=1}^C \frac{R_{cj}}{\max_{c'} R_{c' j}}-1}}, \label{eq:amari}
\end{align}
where $R_{jc} := [\W \A]_{jc}$. This quantity is made zero if and only if $\W \A = \mathbf{P} \mathbf{D}$ for a permutation matrix $\mathbf{P} \in \br{0, 1}^{C \times C}$ and invertible diagonal matrix $\mathbf{D} \in \R^{C \times C}$, naturally inspriring the goal of making~\eqref{eq:amari} as small as possible (see \Cref{fig:amari} for an illustration). Notice that the Amari distance can only be computed if the true mixing matrix $\A$ is known, making it appropriate only for simulation studies. 

For real data experiments, we instead evaluate the method according to the scientific objectives outlined in \Cref{sec:intro}; we evaluate the ability of the estimated sources to retain an explanable dependence with the labels used for supervision. This is measured both quantitatively by prediction metrics such as accuracy and qualitatively using visualizations of the learned sources on held-out data. Code and additional details for reproducibility can be found at \href{https://ronakdm.github.io/_pages/software}{https://ronakdm.github.io/\_pages/software}.

\paragraph{Effect of Multiple Trials}
We evaluate \algoname against various baselines for single-trial (unsupervised) ICA using the Amari distance. We compare to the cumulant-based methods FOBI \citep{cardoso1989source} and JADE \citep{cardoso1993blind}. As is commonly used as a non-Gaussian synthetic data benchmark for ICA, we draw $\s_1, \ldots, \s_N \in \R^{C \times T}$ with independent $\laplace(1)$ entries and $\A$ with standard normal entries, for $(N, C, T) = (80, 10, 1000)$. First, for each trial, we compute the Amari distances $d_1, \ldots, d_N$, where $d_i =  \amari(\W^{-1}_{\mathrm{Base}}(\z_i), \A)$ and $\W_{\mathrm{Base}}(\z_i)$ denotes the output of the baseline consuming the data of the $i$-th trial. We then compared the mean and median Amari distance over trials to $\amari(\W^{-1}_{\mathrm{MultiICA}}, \A)$ for $(n, \tau, \eta_u) = (10, 64, 0.1)$, and averaged the results over 50 random seeds, as shown in the left panel of \Cref{fig:multi}. To understand the degree to which the observed improvement of \algoname over baselines is due to the larger sample size available for it (all $NT$ time points versus $T$ per-trial for the baselines), we compute the Amari distance $\amari(\W^{-1}_{\mathrm{Base}}([\z_1|\cdots|\z_N]), \A)$ where $[\z_1|\cdots|\z_N] \in \R^{C \times N T}$ is the concatenation of the trials along the time axis. As shown in the right panel of \Cref{fig:multi}, \algoname consistently estimates the unmixing matrix with Amari distance two orders of magnitude lower than FOBI, but is outperformed by JADE. Note, however, that both JADE and FOBI are \emph{full batch} methods, and that the time complexity of JADE on concatenated data is $O(C^5)$ per iteration with an $O(N T C^4)$ initialization. In contrast, using batch sizes $n \le N$ for the trials and $\tau \le T$ for the samples, \algoname converges to an approximate solution with an $O\p{n \tau C^3}$ cost per iteration. 

\begin{figure*}[t]
    \centering
    \includegraphics[width=\linewidth]{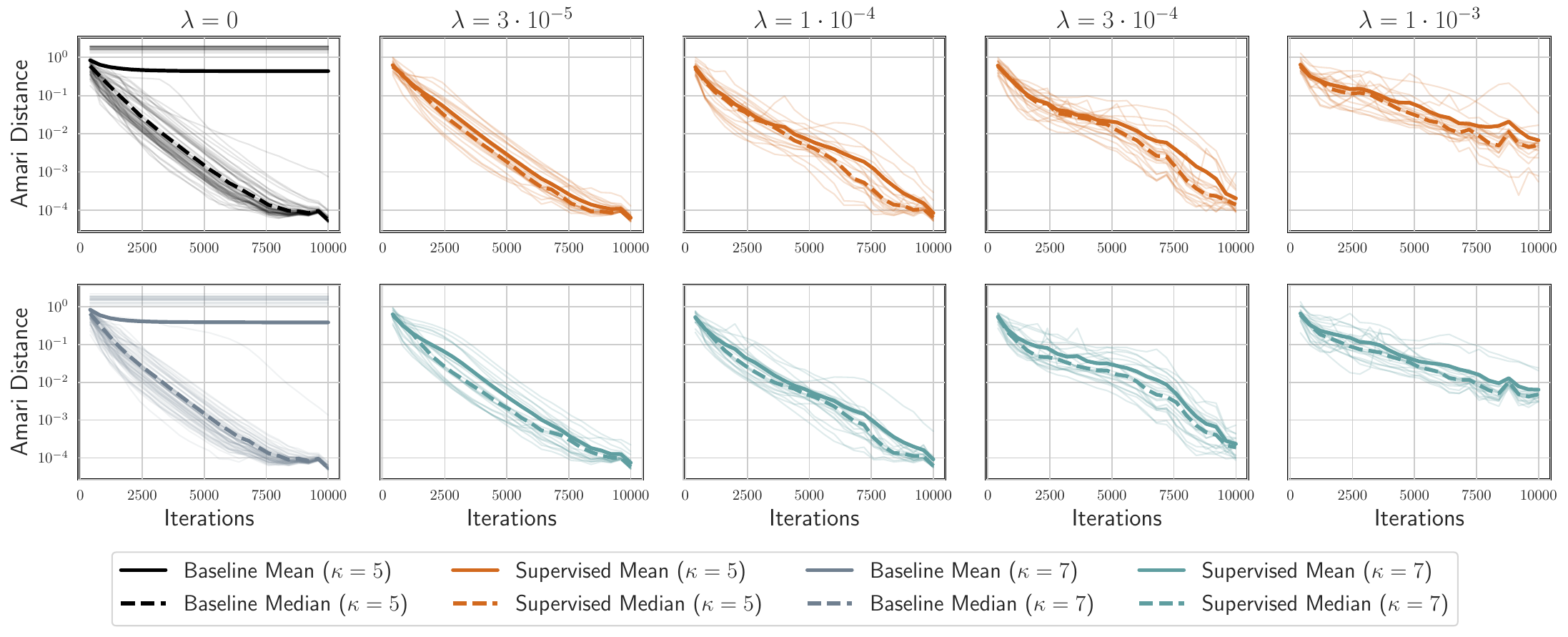}
    \caption{{\bf Results: Supervision.} In each plot, translucent lines depict the trajectories from applying \Cref{algo:multiica} with varying seeds, whereas the solid and dashed lines represent aggregations across seeds. ``Baseline'' refers to the $\lambda = 0$ (unsupervised) runs of \Cref{algo:multiica}, whereas $\kappa$ is the natural logarithm of the condition number of the mixing matrix $\A$.}
    \label{fig:supervision}
\end{figure*}

\begin{figure}[t]
    \centering
    \includegraphics[width=0.8\linewidth]{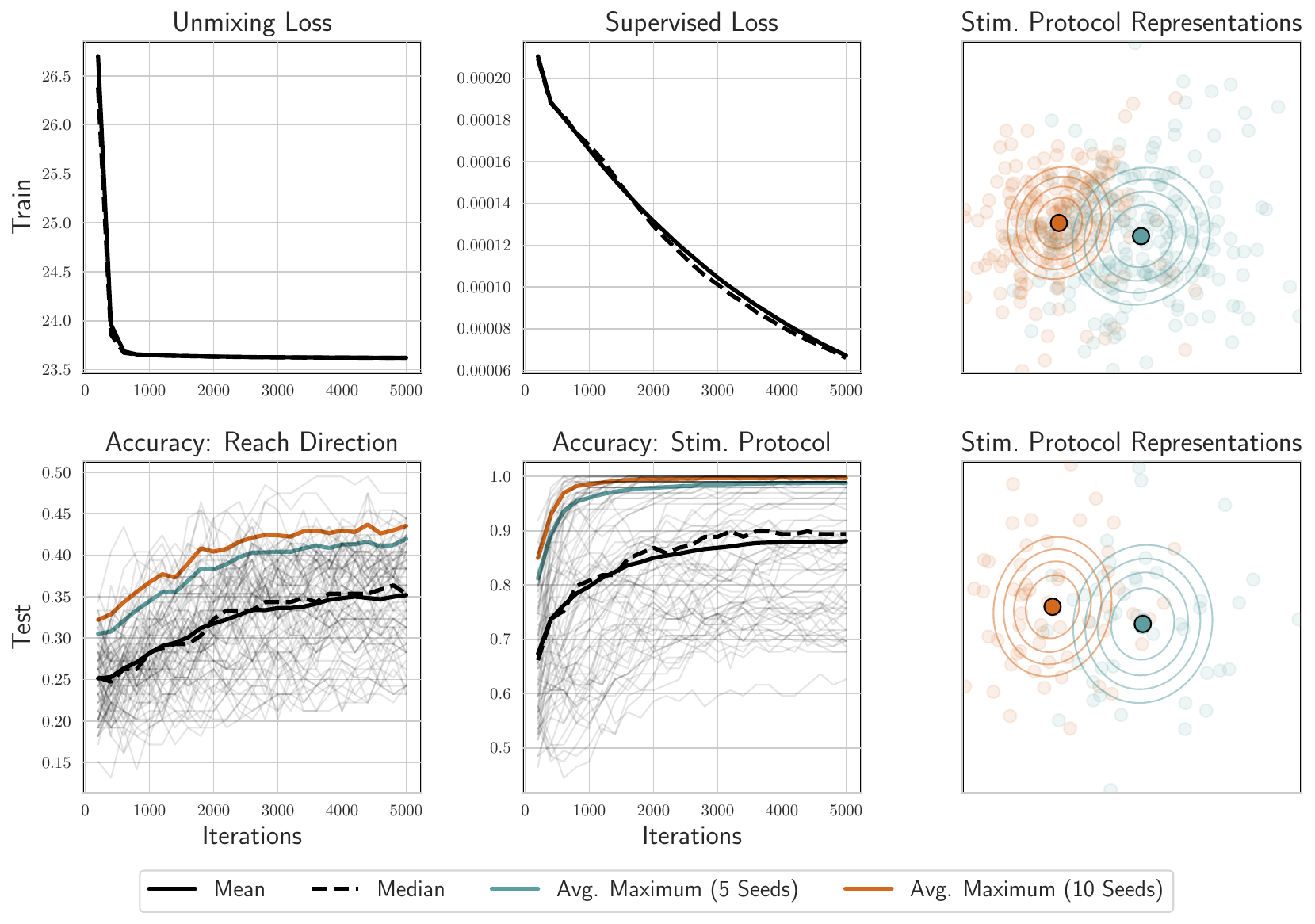}
    \caption{{\bf Results: $\mu$ECoG Data.} The top left panel plots the unmixing objective $\frac{1}{N}\sum_{i=1}^N \ell_0(\W\pow{k},  \z_i)$ whereas the top center panel plots the supervised objective $\frac{\lambda}{N} \sum_{i=1}^N \sum_{m=1}^M  \ell_m((\W_{m \cdot}\pow{k})^\top \z_i, y_{i, m}, \btheta_m\pow{k})$ against the iteration count $k$. 
    The bottom left and center plots show the test accuracy of predicting the reach direction (four balanced classes) and the stimulation protocol (two balanced classes) of the models parametrized by $\btheta\pow{k}_1$ and $\btheta\pow{k}_2$, respectively. The blue (resp.~brown) line represents the average performance of taking the maximum of 5 (resp.~10) seeds for all combinations of 5 (resp.~10) seeds out of 80 total. Finally, the top right and bottom right plots show vector representations of the source corresponding to the stimulation protocol, both on the training trials (top) and held out trials (bottom). The contour lines represent level sets of the Gaussian density fitted to each group of colored points, which indicate whether stimulation was (blue) or was not (brown) applied. }
    \label{fig:reach}
\end{figure}

\paragraph{Effect of Supervision}
Next, we isolate the effect of supervision and the relationship between the hyperparameters $(\lambda, \etau, \etap)$ used in \Cref{algo:multiica}. The sources are similarly constructed with independent $\laplace(1)$ entries $(N, C, T) = (6000, 10, 1000)$. We generate $M = 3$ regression targets using spectrogram features of the independent sources. That is, letting $\phi: \R^{T} \rightarrow \R^{d}$ denote a differentiable feature map. Then, for $i = 1, \ldots, N$ and $m = 1, \ldots, M$, we compute labels $y_{i, m} = \ip{\btheta_m, \phi(\s_{i, m})}$ for $\btheta_m$ generated with standard normal entries. For $\A$, we generate a matrix using the eigenspaces of the Hilbert matrix, with condition number $\exp(\kappa)$ parametrized by the constant $\kappa > 0$. The results for $\kappa \in \br{5, 7}$ and $\lambda$ ranging between $3 \cdot 10^{-5}$ and $1 \cdot 10^{-3}$ are shown in \Cref{fig:supervision} for 80 seeds and batch sizes $n = \tau = 128$. In all cases, we set $\etau = 0.001$, which generated the most stable and fast performance across various seeds, distributions, and settings. A crucial observation in \Cref{fig:supervision} is that for $\lambda = 0$, a non-trivial proportion of trajectories fail (e.g., converge to a local minimum) using only the unsupervised component of the objective for guidance. This effect catastrophically harms the mean performance across seeds. On the other hand, the final Amari distance actually \emph{increases} when the supervision parameter increases, which is likely due to the bias of the jointly learned parameters $(\btheta\pow{k}_m)_{k \geq 0}$. Thus, by applying a small amount of supervision, the mixing matrix can be recovered (down to scaling and permutation) with significantly higher probability. Finally, based on both the synthetic examples and upcoming real data example, we recommend that practitioners simply fix $(\lambda, \etau) = (0.00003, 0.001)$ and tune the model learning rate $\etap$, which is empirically seen to be 1-2 orders of magnitude smaller than $\etau$. Notably, the sequence $\W\pow{k}$ may converge much faster as $k \rightarrow \infty$ than $\btheta\pow{k}$ during the joint learning process.

\paragraph{Application to Neural Data}
To demonstrate the applicability of the \algoname method to neural data, we applied it to micro-electrocorticography ($\mu$ECoG) signals recorded from $C=8$ locations on the posterior parietal cortex of an adult male rhesus macaque performing a reaching task, both with and without external stimulation. The $\mu$ECoG signals were sampled at 1000Hz for 2 seconds, resulting in $T = 2000$ samples. Optogenetic stimulation was used to deactivate neurons that had been genetically modified to respond to light. The behavioral task involved reaching in one of four cardinal directions, with optogenetic deactivation applied in approximately 50\% of the $N = 394$ training trials. The reach direction and the stimulation protocol (whether stimulation was applied) generate $M = 2$ discrete supervised targets on each trial. We use the same predictive model as in \Cref{fig:supervision} and $\lambda = 0.00003$, with batch sizes $(n, \tau) = (32, 64)$.
As described at the start of this section, because there is no known mixing matrix $\A$, we evaluate the quality of $\W$ in \Cref{fig:reach} by illustrating how the learned sources retain information about stimulation and reach direction. 

Regarding the unsupervised and supervised losses, we observe the effect of the unmixing matrix converging much faster than the model parameters. Furthermore, while the unmixing matrix is the ultimate object of interest in our setting, if the supervised model learned by the algorithm is in fact useful for downstream applications, we recommend either taking the best model from multiple seeds (as in \Cref{fig:reach}) or training another supervised model on top of the independent sources. Indeed, the actual classification performance of the learned model may vary significantly (see \Cref{fig:reach}, bottom), even if the loss of the unmixing matrix does not. We find that 5 seeds are sufficient to get close to the best possible performance on average, which corresponds to the blue lines in \Cref{fig:reach}. To illustrate the explainability aspect, we plot vector representations of the learned sources on the right panels of \Cref{fig:reach}. These representations are the first two principal components of the flattened spectrograms of the source corresponding to the stimulation protocol target, showing an affine shift in the representation space associated with stimulation.

\section{Conclusion}\label{sec:conclusion}
We proposed an ICA algorithm that incorporates supervision from multiple trials with auxiliary target variables to improve the learning trajectory of the unmixing matrix. The algorithm combines proximal gradient-type updates for the unmixing matrix with generic, backpropagation-based learning for a supervised model that is jointly learned with the unmixing matrix. We illustrated the method on synthetic and real data, and discussed practical hyperparameter selection for users. Future work includes theoretical convergence analysis and incorporating orthogonality constraints, as in FastICA.

\paragraph{Acknowledgments}
This work was supported by NSF DMS-2023166, CCF-2019844, DMS-2134012, NIH. The authors are grateful to E. Shea-Brown and A. Shojaie for fruitful discussions.

\bibliographystyle{abbrvnat}
\bibliography{bib}

\appendix

\section{Descent Guarantees}\label{sec:a:descent}
\subsection{Derivations}

We derive the update given in \Cref{prop:update}, starting with a technical lemma related to objectives that result from linear combinations of quadratic and log-determinant terms.
\begin{lemma}\label{lem:update}
    Consider a function $h: \R^c \rightarrow \R$ of the form
    \begin{align*}
        h(\r) := \frac{1}{2} \r^\top \mathbf{K} \r - \log \abs{r_j} - \ip{\mathbf{b}, \r},
    \end{align*}
    where $\mathbf{K} \in \R^{C \times C}$ is invertible, $b \in \R^C$, and we define $h(\r) = +\infty$ when $r_j = 0$ for $j \in [C]$. Then, the minimizer of $h$, denoted $\r^\star$ is given by
    \begin{align*}
        \r^\star := \mathbf{K}^{-1} \p{\frac{1}{r_{j}^\star} \e_c + \mathbf{b}},
    \end{align*}
    for $r_{j}^\star = \sqrt{\mathbf{K}^{-1}_{jj} + \tfrac{1}{4}\p{\mathbf{K}^{-1} \mathbf{b}}_j^2} +\tfrac{1}{2}(\mathbf{K}^{-1}\mathbf{b})_j$.
\end{lemma}
\begin{proof}
    Optimizing the function $h$ over $r_j > 0$ or $r_j < 0$ will result in the same gradient formula and a strongly convex function. Indeed, we have that
    \begin{align*}
        \grad h(\r) = \mathbf{K}  \r - \frac{1}{r_j} \e_j - \mathbf{b} = 0
    \end{align*}
    which implies the relationship
    \begin{align}
        \r = \mathbf{K}^{-1} \p{\frac{1}{r_j} \e_j + \mathbf{b}}. \label{eq:multi_minimizer}
    \end{align}
    Given $r_j$, it is easy to compute the rest of $\r$. By looking at the $j$-th coordinate of the relationship above, we have that
    \begin{align*}
        r_j^2  &= \mathbf{K}^{-1}_{jj} + r_j (\mathbf{K}^{-1} \mathbf{b})_j\\
        \iff r_j  &= \sqrt{\mathbf{K}^{-1}_{jj} + \tfrac{1}{4}(\mathbf{K}^{-1} \mathbf{b})_j^2} + \tfrac{1}{2}(\mathbf{K}^{-1} \mathbf{b})_j
    \end{align*}
    because we fixed the convention that $r_j > 0$. Plugging this back into the equation~\eqref{eq:multi_minimizer} completes the proof.
\end{proof}

We may now prove \Cref{prop:update}, which forms the basis for line~\ref{line:update} of \Cref{algo:multiica}.
\update*
\begin{proof}
    First, expanding the Frobenius norm term, observe the relationship
    \begin{align}
        F\pow{k, c-1}(\W, \btheta\pow{k}, \U\pow{k}) = \frac{1}{2}\sum_{c=1}^C \W_{c \cdot}^\top \A_c \W_{c \cdot} + \ip{\B, \W} + L(\W) + \const(\W), \label{eq:reduced_objective}
    \end{align}
    for $\A_c := \A_c\pow{k} + \etau^{-1}\mathbf{I}$ and $\B = \lambda\B\pow{k}_c - \etau^{-1} \W\pow{k, c-1}$. Then, evaluating~\eqref{eq:reduced_objective} at $\overline{\W}\pow{k, j-1} := (\e_{1:j-1}, \r, \e_{j+1:C})^\top \W\pow{k, j-1}$, it is clear that
    \begin{align*}
        \frac{1}{2}\sum_{c=1}^C (\overline{\W}_{c \cdot}\pow{k, j-1})^\top \A_c \overline{\W}_{c \cdot}\pow{k, j-1} = \frac{1}{2}\r^\top \W\pow{k, j-1} \A_j\pow{k} (\W\pow{k, j-1})^\top \r + \const(\r)
    \end{align*}
    and that
    \begin{align*}
        -\logdet{\overline{\W}\pow{k, j-1}} = -\log\abs{r_{c}} +  \const(\r).
    \end{align*}
    Similarly, observe that
    \begin{align*}
        \ipsmall{\B, \overline{\W}\pow{k, j-1}}
        &= \tr\p{(\overline{\W}\pow{k, j-1})^\top \B}\\
        &= \tr\p{\B (\overline{\W}\pow{k, j-1})^\top}\\
        &= \tr\p{\B  (\W\pow{k, j-1})^\top (\e_{1:j-1}, \r, \e_{j+1:C}) }\\
        &= (\B(\W\pow{k, j-1})^\top \r)_j + \const(\r)\\
        &= \ipsmall{\W\pow{k, j-1} \B_{j\cdot}, \r} + \const(\r).
    \end{align*}
    Thus, using the definitions of $\mathbf{K}$ and $\mathbf{b}$ in the statement, $\r_j$ is selected by optimizing 
    \begin{align*}
        h(\r) := \frac{1}{2} \r^\top \mathbf{K} \r - \log \abs{r_j} - \ip{\mathbf{b}, \r}
    \end{align*}
    over $\r \in \R^C$. Apply \Cref{lem:update} to complete the proof.    
\end{proof}

\subsection{Monotonicity}

This subsection is dedicated to the proof of \Cref{lem:monotonicity}, which is restated at the end. As a broad outline, the proof proceeds via the sequence of inequalities
\begin{align}
    F_\mu(\W\pow{k-1}, \btheta\pow{k-1}, \blue{\U\pow{k-1}}) &\geq F_\mu(\W\pow{k-1}, \red{\btheta\pow{k-1}}, \blue{\U\pow{k}})\notag\\
    &\geq \underbrace{F_\mu(\green{\W\pow{k-1}}, \red{\btheta\pow{k}}, \U\pow{k})}_{F_\mu(\green{\W\pow{k, 0}}, \btheta\pow{k}, \U\pow{k})} \geq \ldots \geq \underbrace{F_\mu(\green{\W\pow{k}}, \btheta\pow{k}, \U\pow{k})}_{F_\mu(\green{\W\pow{k, C}}, \btheta\pow{k}, \U\pow{k})}, \label{eq:monotonicity}
\end{align}
where the colors indicate what changes at each step and the underbraces use the notation of the cyclic update in \Cref{algo:multiica}. The first inequality is immediate because
\begin{align*}
    F_\mu(\W\pow{k-1}, \btheta\pow{k-1}, \U\pow{k-1}) \geq \min_{\U \geq \zeros_{N \times C \times T}} F_\mu(\W\pow{k-1}, \btheta\pow{k-1}, \U) = F_\mu(\W\pow{k-1}, \btheta\pow{k-1}, \U\pow{k}).
\end{align*}
The remainder of the work will be to apply the smoothness conditions so that the gradient-based updates result in descent for small enough learning rates. 

We may easily see that for fixed $\y = (y_1, \ldots, y_m)$ and $\btheta$ and the function $G$ defined by
\begin{align*}
    G(\W) &= \frac{1}{N} \sum_{i=1}^N \sum_{m=1}^M \ell_m(\W_{m \cdot}^\top \z_i, y_m, \btheta_m)\\
    \grad G(\W) &=
    \frac{1}{N} \sum_{i=1}^N
    \begin{bmatrix}
        \grad_{\s} \ell_1(\W_{1 \cdot}^\top \z_i, y_1, \btheta_1)^\top \\
        \vdots\\
        \grad_{\s} \ell_M(\W_{M \cdot}^\top \z_i, y_M, \btheta_M)^\top \\
        \zeros_{(C - M) \times C}
    \end{bmatrix}\z_i^\top
\end{align*}
it holds via the assumptions of \Cref{lem:monotonicity} and the triangle inequality that
\begin{align}
    \norm{\grad G(\W) - \grad G(\W')}_{\Fro} &\leq  \p{\textstyle\tfrac{1}{N}\sum_{i=1}^N \norm{\z_i}_{2, 2}^2} \sum_{m=1}^M L_m\norm{\W_{m \cdot} - \W'_{m \cdot}}_{2}\notag\\
    &\leq  \p{\textstyle\tfrac{1}{N}\sum_{i=1}^N \norm{\z_i}_{2, 2}^2} \p{\sum_{m=1}^M L_m^2}^{1/2}\p{\sum_{m=1}^M \norm{\W_{m \cdot} - \W'_{m \cdot}}_{2}^2}^{1/2} \notag\\
    &\leq  \underbrace{\p{\textstyle\tfrac{1}{N}\sum_{i=1}^N \norm{\z_i}_{2, 2}^2} \sqrt{\textstyle\sum_{m=1}^M L_m^2}}_{L_{\W}} \norm{\W - \W'}_{\Fro}.\label{eq:smoothness}
\end{align}
Thus, we used the scaled Lipschitz constant $L_{\W}$ in the proof below.

\monotonicity*
\begin{proof}
    The result is shown by proving the inequalities in~\eqref{eq:monotonicity}. Note the equivalence
    \begin{align*}
        F_\mu(\W\pow{k-1}, \red{\btheta\pow{k-1}},\U\pow{k}) \geq F_\mu(\W\pow{k-1}, \red{\btheta\pow{k}}, \U\pow{k}) \iff H(\btheta\pow{k-1}) \geq H(\btheta\pow{k})
    \end{align*}
    for the function $H(\btheta) := \frac{1}{N} \sum_{i=1}^N \sum_{m=1}^M \ell_m(\W\pow{k-1} \z_i, y_m, \btheta_m) + \frac{\mu}{2}\norm{\btheta}_2^2$. By assumption, we have that $\grad H$ is $(L_{\btheta} + \mu)$-Lipschitz continuous. Furthermore, the update rule can be summarized as
    \begin{align*}
        \btheta\pow{k} = \btheta\pow{k-1} - \etap \grad H(\btheta\pow{k-1}),
    \end{align*}
    or in other words, $\grad H(\btheta\pow{k-1}) = \g\pow{k} + \mu \btheta\pow{k-1}$. Apply \citet[Theorem 2.1.5, Eq. (2.1.9)\footnote{Note that the inequality employed does not require convexity of the objective, despite the theorem assumptions.}]{nesterov2018lectures} to achieve
    \begin{align*}
        H(\btheta\pow{k}) &\leq H(\btheta\pow{k-1}) + \ipsmall{\grad H(\btheta\pow{k-1}), \btheta\pow{k} - \btheta\pow{k-1}} + \frac{L_{\btheta} + \mu}{2}\norm{\btheta\pow{k} - \btheta\pow{k-1}}_2^2\\
        &= H(\btheta\pow{k-1}) - \etap\underbrace{\p{1 - \frac{(L_{\btheta} + \mu) \etap}{2}}}_{\geq 1/2} \norm{\grad H(\btheta\pow{k-1})}_2^2 \leq H(\btheta\pow{k-1}),
    \end{align*}
    where the last inequality follows under the condition that $\etap \leq \frac{1}{L_{\btheta} + \mu}$. For the second part of the proof, we show that for $c = 1, \ldots, C$, it holds that
    \begin{align}
         F_\mu(\green{\W\pow{k, c-1}}, \btheta\pow{k}, \U\pow{k}) \geq F_\mu(\green{\W\pow{k, c}}, \btheta\pow{k}, \U\pow{k}),\label{eq:cyclic_ineq}
    \end{align}
    which can be chained to show the desired result. Because $\btheta\pow{k}$ is fixed, we show that $F(\W\pow{k, c}, \btheta\pow{k}, \U\pow{k}) \leq F(\W\pow{k, c-1}, \btheta\pow{k}, \U\pow{k})$. To do so, we use the smoothness constant computation~\eqref{eq:smoothness} and apply \citet[Theorem 2.1.5, Eq. (2.1.9)]{nesterov2018lectures} once again to achieve
    \begin{align*}
        &F(\W\pow{k, c}, \btheta\pow{k}, \U\pow{k}) + \frac{1}{2\etau}\norm{\W\pow{k, c} - \W\pow{k, c-1}}_{\Fro}^2 \\
        &= L(\W\pow{k, c})  + \frac{1}{2}\sum_{j=1}^C (\W_{j\cdot}\pow{k, c})^\top \A_c\pow{k} \W_{j\cdot}\pow{k, c} + \frac{1}{N}\sum_{i, c, t} f(U_{i, c, t}\pow{k}) \\
        &\quad + \frac{\lambda}{N}\sum_{i=1}^N \sum_{m=1}^M  \ell_m((\W_{m \cdot}\pow{k, c})^\top \z_i, y_{i, m}, \btheta_m\pow{k}) + \frac{1}{2\etau}\norm{\W\pow{k, c} - \W\pow{k, c-1}}_{\Fro}^2\\
        &\leq L(\W\pow{k, c})  + \frac{1}{2}\sum_{j=1}^C (\W_{j\cdot}\pow{k, c})^\top \A_c\pow{k} \W_{j\cdot}\pow{k, c} + \frac{1}{N}\sum_{i, c, t} f(U_{i, c, t}\pow{k}) \\
        &\quad + \frac{\lambda}{N}\sum_{i=1}^N \sum_{m=1}^M  \ell_m((\W_{m \cdot}\pow{k, c-1})^\top \z_i, y_{i, m}, \btheta_m\pow{k})\\
        &\quad + \lambda\ipsmall{\B\pow{k}_c, \W\pow{k, c} - \W\pow{k, c-1}} + \p{\frac{\lambda L_{\W}}{2} + \frac{1}{2\etau}}\norm{\W\pow{k, c} - \W\pow{k, c-1}}_{\Fro}^2\\
        &\leq F(\W\pow{k, c-1}, \btheta\pow{k}, \U\pow{k}) + \frac{\lambda L_{\W}}{2}\norm{\W\pow{k, c} - \W\pow{k, c-1}}_{\Fro}^2,
    \end{align*}
    where in the last inequality we used the definition of $\W\pow{k, c}$ as the minimizer of the objective~\eqref{eq:W_update0}, for which $\W\pow{k, c-1}$ is feasible. Thus, the inequality \eqref{eq:cyclic_ineq} is satisfied when $\etau \leq \frac{1}{2\lambda L_{\W}}$, completing the proof of monotonicity. If $F$ is bounded from below, then $F_\mu$ is as well; applying the monotone convergence theorem to the sequence $(F(\W\pow{k}, \btheta\pow{k}, \U\pow{k}))_{k=1}^\infty$ proves the second claim and completes the proof.
\end{proof}

\section{Convergence to a Stationary Point}\label{sec:a:stationarity}

In this section, we analyze a modification to the algorithm (see~\eqref{eq:aux_update2}) and a set of technical tools to establish convergence to a stationary point of the objective. To outline the formal statement of \Cref{prop:stationarity} and proof, we introduce some additional notation and definitions. For any extended real-valued function $\Psi: \R^D \rightarrow \R \cup \br{+\infty}$, finite at $\x \in \R^D$, we define the \emph{Fr\'{e}chet subdifferential} at $\x$ as the set
\begin{align*}
    \partial \Psi(\x) := \br{\x^* \in \R^D: \liminf_{\v \rightarrow \x} \frac{\Psi(\v) - \Psi(\x) - \ip{\x^*, \v - \x}}{\norm{\v - \x}_2} \geq 0},
\end{align*}
which may be empty. An element of this set will be called a \emph{Fr\'{e}chet subgradient}. In our case, we consider $\x = (\W, \btheta, \U)$ and $D = C^2 + d + NCT$ (by flattening the tensors appropriately), and the objective~\eqref{eq:full_objective_frechet}.
We will define \emph{stationarity} of a sequence $(\W\pow{k}, \btheta\pow{k}, \U\pow{k})_{k \geq 0}$ as there existing a sequence of Fr\'{e}chet subgradients $G\pow{k} = (G\pow{k}_{\U}, G\pow{k}_{\btheta}, G\pow{k}_{\W})$ such that $\smallnorm{G\pow{k}}_{\Fro} \rightarrow 0$ as $k \rightarrow \infty$.

Let us outline the proof. First, given the result of \Cref{lem:monotonicity}, it is clear that $F_\mu(\W\pow{k-1}, \btheta\pow{k-1}, \U\pow{k-1}) - F_\mu(\W\pow{k}, \btheta\pow{k}, \U\pow{k}) \rightarrow 0$ as $k \rightarrow \infty$ due to monotonicity and boundedness of the sequence. To achieve a stationarity result, we first prove a descent inequality of the form
\begin{align*}
    F_\mu(\W\pow{k}, \btheta\pow{k}, \U\pow{k}) &\leq F_\mu(\W\pow{k-1}, \btheta\pow{k-1}, \U\pow{k-1})  - \frac{1}{2\etaa} \norm{ \U\pow{k} -  \U\pow{k-1}}^2_{\Fro}\\
    &\quad - \frac{1}{2\etap} \norm{ \btheta\pow{k} -  \btheta\pow{k-1}}^2_{2}- \frac{1}{4 \etau}  \norm{\W\pow{k} - \W\pow{k-1}}_{\Fro}^2,
\end{align*}
which, due to the convergence of the function values, provides the asymptotic result
\begin{align}
    I\pow{k} := \norm{\U\pow{k} -  \U\pow{k-1}}^2_{\Fro} + \norm{\btheta\pow{k} -  \btheta\pow{k-1}}^2_{2} + \norm{\W\pow{k} - \W\pow{k-1}}_{\Fro}^2 \rightarrow 0\label{eq:iter_converge}
\end{align}
as $k \rightarrow \infty$. In the second step, we compute a particular Fr\'{e}chet subgradient $(G\pow{k}_{\U}, G\pow{k}_{\btheta}, G\pow{k}_{\W}) \in \partial \Psi(\U\pow{k}, \btheta\pow{k}, \W\pow{k})$, and show that
\begin{align}
    \norm{(G\pow{k}_{\U}, G\pow{k}_{\btheta}, G\pow{k}_{\W})}_{\Fro}^2 \leq C_0 I\pow{k} \label{eq:subdiff_ineq}
\end{align}
for a constant $C_0 \geq 0$. Given that $I\pow{k} \rightarrow 0$, it holds that this subgradient will also converge to zero, providing convergence of the sequence $(\W\pow{k}, \btheta\pow{k}, \U\pow{k})$ to a stationary point. These two broad steps comprise the next two subsections, which culminates in \Cref{prop:subdiff_bd2}.
The first is used in other analyses of non-convex optimization trajectories \citep{bolte2014proximal}, which is the boundedness of the iterates.
\begin{assumption}\label{asm:bdd1}
    There exists a constant $B > 0$ such that
    \begin{align*}
        \textstyle\max_k \max \br{\smallnorm{\W\pow{k}}_{\Fro}, \smallnorm{\btheta\pow{k}}_2, \smallnorm{\U\pow{k}}_{\Fro}} \leq B.
    \end{align*}
\end{assumption}
The second assumption will be similar, although instead of boundedness of the iterates, we require that the $\W\pow{k}$ trajectory does not become ``too close'' to singularity. In this case, we can place a smoothness condition on the log-determinant function.
\begin{assumption}\label{asm:bdd2}
    For a constant $\sigma > 0$, define the set $\mathcal{W}_\sigma$, which is a subset of the invertible matrices, using the following condition. For any $\W, \W' \in \mathcal{W}_\sigma$ and $c = 1, \ldots, C$,  it holds that
    \begin{align*}
        \norm{\grad_{\W_{c \cdot}} L(\W) - \grad_{\W_{c \cdot}'} L(\W')} \leq \frac{1}{\sigma}\norm{\W - \W'}_{\Fro}. 
    \end{align*}
    Assume that there exists $\sigma > 0$ such that $\mathcal{W}_\sigma$ is non-empty and that $\W\pow{k, c} \in \mathcal{W}_\sigma$ for $k = 1, \ldots, K$ and $c = 1, \ldots, C$.
\end{assumption}
Another interpretation is that we confine the matrices to a sub-level set of the negative log-determinant function. Here,  $\sigma$ can also be interpreted as (a reparametrization of) a minimum singular value constant. The final assumption is simply a smoothness assumption on the function $f$, which first appears in the variational form~\eqref{eq:supergaussian} and consequently the full objective~\eqref{eq:full_objective}.
\begin{assumption}\label{asm:bdd3}
    Assume that $f$ is differentiable, and for that any $u, v \geq 0$, it holds that
    \begin{align*}
        \abs{f'(u) - f'(v)} \leq L_{\U} \abs{u - v}.
    \end{align*}
\end{assumption}

\subsection{Descent Inequalities}

Here, we reuse the same proof of \Cref{lem:monotonicity}, but retain the additional non-positive terms that were dropped as slack for cancellation. While the analysis for $\U\pow{k}$ and $\btheta\pow{k}$ are nearly identical to the previous results, we must chain the inequalities for the inner loop of the $\W\pow{k}$ update so that they sum to the desired result.
\begin{lemma}\label{lem:descent_stat}
    Under the conditions of \Cref{lem:monotonicity}, the folloing three inequalities hold:
    \begin{align}
        F_\mu(\W\pow{k-1}, \btheta\pow{k-1}, \U\pow{k}) &\leq F_\mu(\W\pow{k-1}, \btheta\pow{k-1}, \U\pow{k-1}) - \frac{1}{2\etaa} \norm{ \U\pow{k} -  \U\pow{k-1}}^2_{\Fro},\\
        F_\mu(\W\pow{k-1}, \btheta\pow{k}, \U\pow{k}) &\leq F_\mu(\W\pow{k-1}, \btheta\pow{k-1}, \U\pow{k}) - \frac{1}{2\etap} \norm{ \btheta\pow{k} -  \btheta\pow{k-1}}^2_{2},\\
        F_\mu(\W\pow{k}, \btheta\pow{k}, \U\pow{k}) &\leq F_\mu(\W\pow{k-1}, \btheta\pow{k}, \U\pow{k}) - \frac{1}{4 \etau}  \norm{\W\pow{k} - \W\pow{k-1}}_{\Fro}^2.
    \end{align}
\end{lemma}
\begin{proof}
    The first inequality follows immediately from~\eqref{eq:aux_update2}. By repeating the argument in the proof of \Cref{lem:monotonicity}, we may achieve the two inequalities
    \begin{align*}
        F_\mu(\W\pow{k-1}, \btheta\pow{k}, \U\pow{k}) &\leq F_\mu(\W\pow{k-1}, \btheta\pow{k-1}, \U\pow{k}) - \frac{1}{\etap}\p{1 - \frac{(L_{\btheta} + \mu) \etap}{2}} \norm{\btheta\pow{k} -  \btheta\pow{k-1}}^2_{\Fro}\\
        F(\W\pow{k, c}, \btheta\pow{k}, \U\pow{k}) &\leq F(\W\pow{k, c-1}, \btheta\pow{k}, \U\pow{k}) + \frac{1}{2}\p{\lambda L_{\W} - \frac{1}{\etau}}\norm{\W\pow{k, c} - \W\pow{k, c-1}}_{\Fro}^2
    \end{align*}
    for every step of the inner loop. Using $\etap \leq \frac{1}{L_{\btheta} + \mu}$ above grants the second inequality in the statement of the result. For the third inequality, using that $\etau \leq \frac{1}{2\lambda L_{\W}}$, we have that
    \begin{align*}
        F(\W\pow{k, c}, \btheta\pow{k}, \U\pow{k}) &\leq F(\W\pow{k, c-1}, \btheta\pow{k}, \U\pow{k}) - \frac{1}{4 \etau}\norm{\W\pow{k, c} - \W\pow{k, c-1}}_{\Fro}^2\\
        &= F(\W\pow{k, c-1}, \btheta\pow{k}, \U\pow{k}) - \frac{1}{4 \etau}\norm{\W\pow{k, c}_{c \cdot} - \W\pow{k, c-1}_{c \cdot}}_{2}^2\\
        &= F(\W\pow{k, c-1}, \btheta\pow{k}, \U\pow{k}) - \frac{1}{4 \etau}\norm{\W\pow{k}_{c \cdot} - \W\pow{k-1}_{c \cdot}}_{2}^2,
    \end{align*}
    where the equalities follow from the fact that $\W\pow{k, c}_{j \cdot} = \W\pow{k, c-1}_{j \cdot}$ for $j \neq c$ due to the row constraints and that $\W\pow{k, c}_{c \cdot} = \W\pow{k}_{c \cdot}$ and $\W\pow{k, c-1}_{ c \cdot} = \W\pow{k-1}_{c \cdot}$ due to the order of the updates. By chaining this inequality for $c = C, \ldots, 1$, we have that
    \begin{align*}
        F_\mu(\W\pow{k}, \btheta\pow{k}, \U\pow{k}) &= F_\mu(\W\pow{k, C}, \btheta\pow{k}, \U\pow{k})\\
        &\leq F_\mu(\W\pow{k, 0}, \btheta\pow{k}, \U\pow{k}) - \frac{1}{4 \etau} \sum_{c=1}^C \norm{\W\pow{k}_{c \cdot} - \W\pow{k-1}_{c \cdot}}_{2}^2\\
        &= F_\mu(\W\pow{k -1}, \btheta\pow{k}, \U\pow{k}) -\frac{1}{4 \etau} \norm{\W\pow{k} - \W\pow{k-1}}_{\Fro}^2.
    \end{align*}
\end{proof}
By applying the three inequalities from \Cref{lem:descent_stat}, we achieve the descent guarantee described as the first step in the proof outline. The convergence~\eqref{eq:iter_converge} follows from the rearrangement
\begin{align*}
     \frac{1}{2\max\br{\etaa, \etap, \etau / 2}} I\pow{k} &\leq \frac{1}{2\etaa} \norm{ \U\pow{k} -  \U\pow{k-1}}^2_{\Fro} +  \frac{1}{2\etap} \norm{ \btheta\pow{k} -  \btheta\pow{k-1}}^2_{2} +  \frac{1}{4 \etau}  \norm{\W\pow{k} - \W\pow{k-1}}_{\Fro}^2\\
     &\leq F_\mu(\W\pow{k-1}, \btheta\pow{k-1}, \U\pow{k-1}) - F_\mu(\W\pow{k}, \btheta\pow{k}, \U\pow{k}),
\end{align*}
where the right-hand side converges to zero. 

\subsection{Subdifferential Inequalities}

We now establish~\eqref{eq:subdiff_ineq}, in two parts.
We introduce some additional notation in this section. Decompose the objective further via
\begin{align*}
    \Psi(\U, \btheta, \W) &= F_\mu(\W, \btheta, \U) + \iota_+(\U) = \Lunsup(\W, \U) + \Lsup(\W, \btheta) + \iota_+(\U),
\end{align*}
for
\begin{align*}
    \Lunsup(\W, \U) &= L(\W) + \frac{1}{N}\sum_{i=1}^N \sum_{c=1}^C \sum_{t=1}^T \sbr{\frac{1}{2} U_{i, c, t} [\W \z_i]_{c, t}^2 + f(U_{i, c, t})} \\
    \Lsup(\W, \btheta) &= \frac{\lambda}{N}\sum_{i=1}^N \sum_{m=1}^M  \ell_m(\W_{m \cdot}^\top \z_i, y_{i, m}, \btheta_m) + \frac{\mu}{2}\norm{\btheta}_2^2.
\end{align*}
For the remaining results, we first provide a formula for a Fr\'{e}chet subgradient. We then prove the inequality~\eqref{eq:subdiff_ineq} to achieve the desired result. In the first part, we identify the Fr\'{e}chet subgradient that can be used in~\eqref{eq:subdiff_ineq}.
\begin{lemma}\label{lem:subdiff_bd1}
    Define
    \begin{align*}
        G\pow{k}_{\U} &= \frac{1}{\etaa} \p{\U\pow{k-1} - \U\pow{k}} + \grad_{\U} F_\mu(\W\pow{k}, \btheta\pow{k}, \U\pow{k}) - \grad_{\U} F_\mu(\W\pow{k-1}, \btheta\pow{k-1}, \U\pow{k}) \\
        G\pow{k}_{\btheta} &= \frac{1}{\etap}\p{\btheta\pow{k-1} - \btheta\pow{k}} + \grad_{\btheta} F_\mu(\W\pow{k}, \btheta\pow{k}, \U\pow{k}) - \grad_{\btheta} F_\mu(\W\pow{k-1}, \btheta\pow{k-1}, \U\pow{k})\\
        G\pow{k}_{\W} &= (G\pow{k}_{\W_{1 \cdot}}, \ldots, G\pow{k}_{\W_{C \cdot}})^\top
    \end{align*}
    for
    \begin{align*}
        G\pow{k}_{\W_{c \cdot}} &= \grad_{\W_{c \cdot}} L(\W\pow{k}) - \grad_{\W_{c \cdot}} L(\W\pow{k, c}) \\
        &\quad + \frac{\lambda}{N} \sum_{i=1}^N \grad_{\W_{c \cdot}}  \sbr{\ell_c((\W_{c \cdot}\pow{k})^\top \z_i, y_{i, c}, \btheta_c\pow{k}) - \ell_c((\W_{c \cdot}\pow{k-1})^\top \z_i, y_{i, c}, \btheta_c\pow{k})}\\
        &\quad + \frac{1}{2\etau} (\W\pow{k-1}_{c \cdot} - \W\pow{k}_{c \cdot}),
    \end{align*}
    where the second line is interpreted as zero when $c > M$.
    Then, 
    \begin{align*}
        \p{G\pow{k}_{\U}, G\pow{k}_{\btheta}, G\pow{k}_{\W}} \in \partial \Psi(\U\pow{k}, \btheta\pow{k}, \W\pow{k}).
    \end{align*}
\end{lemma}
\begin{proof}
    By~\eqref{eq:aux_update2}, it holds by the optimality of $\U\pow{k}$ that
    \begin{align*}
        \grad_{\U} F_\mu(\W\pow{k-1}, \btheta\pow{k-1}, \U\pow{k}) + \frac{1}{\etaa} \p{\U\pow{k}- \U\pow{k-1}} + \mathbf{S}\pow{k} = \zeros_{N \times C \times T}
    \end{align*}
    for some $\mathbf{S}\pow{k} \in \partial \iota_+(\U\pow{k})$. It also holds that
    \begin{align*}
        \grad_{\U} F_\mu(\W\pow{k}, \btheta\pow{k}, \U\pow{k}) + \mathbf{S}\pow{k} \in \partial_{\U} \Psi(\U\pow{k}, \btheta\pow{k}, \W\pow{k}),
    \end{align*}
    which, combined with the above, yields
    \begin{align*}
        G\pow{k}_{\U} &= \frac{1}{\etaa} \p{\U\pow{k-1} - \U\pow{k}} + \grad_{\U} F_\mu(\W\pow{k}, \btheta\pow{k}, \U\pow{k}) - \grad_{\U} F_\mu(\W\pow{k-1}, \btheta\pow{k-1}, \U\pow{k})  \\
        &\in  \partial_{\U} \Psi(\U\pow{k}, \btheta\pow{k}, \W\pow{k}).
    \end{align*}
    By the definition of $\btheta\pow{k}$, it holds that
    \begin{align*}
        \frac{1}{\etap}\p{\btheta\pow{k} - \btheta\pow{k-1}} - \grad_{\btheta} F_\mu(\W\pow{k-1}, \btheta\pow{k-1}, \U\pow{k}) = \zeros_{d}.
    \end{align*}
    It holds trivially that $\grad_{\btheta} F_\mu(\W\pow{k}, \btheta\pow{k}, \U\pow{k}) + \zeros_{d}\in \partial_{\btheta} \Psi(\U\pow{k}, \btheta\pow{k}, \W\pow{k})$ and substituting the expression for $\zeros_{d}$ then proves that $G\pow{k}_{\btheta} \in \partial_{\btheta} \Psi(\U\pow{k}, \btheta\pow{k}, \W\pow{k})$.
    Finally, for any vector $\w \in \R^C$, we use $\W^{k, c-1}(\w)$ to indicate the matrix $\W^{k, c-1}$ with its $c$-th row replaced by $\w$.
    Due to the definition
    \begin{align*}
        \W\pow{k}_{c \cdot } = \argmin_{\w \in \R^C} \Lunsup(\W\pow{k, c-1}(\w), \U\pow{k}) + \ipsmall{\grad_{\W_{c \cdot}} \Lsup(\W\pow{k, c-1}, \btheta\pow{k}), \w} + \frac{1}{2\etau} \norm{\w - \W\pow{k, c-1}_{c \cdot}}_{2}^2,
    \end{align*}
    we have the identity
    \begin{align*}
        \zeros_C &= \grad_{\W_{c \cdot}} \Lunsup(\W\pow{k, c}, \U\pow{k}) + \grad_{\W_{c \cdot}} \Lsup(\W\pow{k, c-1}, \btheta\pow{k}) + \frac{1}{2\etau} (\W\pow{k, c}_{c \cdot} - \W\pow{k, c-1}_{c \cdot})\\
        &= \grad_{\W_{c \cdot}} \Lunsup(\W\pow{k, c}, \U\pow{k}) + \grad_{\W_{c \cdot}} \Lsup(\W\pow{k, c-1}, \btheta\pow{k}) + \frac{1}{2\etau} (\W\pow{k}_{c \cdot} - \W\pow{k-1}_{c \cdot}).
    \end{align*}
    Then, because
    \begin{align*}
        \grad_{\W_{c \cdot}} \Psi(\U\pow{k}, \btheta\pow{k}, \W\pow{k}) = \grad_{\W_{c \cdot}} \Lunsup(\W\pow{k}, \U\pow{k}) + \grad_{\W_{c \cdot}} \Lsup(\W\pow{k}, \btheta\pow{k}),
    \end{align*}
    we have that
    \begin{align*}
        G\pow{k}_{\W_{c \cdot}} &= \grad_{\W_{c \cdot}} \Lunsup(\W\pow{k}, \U\pow{k}) - \grad_{\W_{c \cdot}} \Lunsup(\W\pow{k, c}, \U\pow{k}) \\
        &\quad + \grad_{\W_{c \cdot}} \Lsup(\W\pow{k}, \btheta\pow{k}) - \grad_{\W_{c \cdot}} \Lsup(\W\pow{k, c-1}, \btheta\pow{k})\\
        &\quad + \frac{1}{2\etau} (\W\pow{k-1}_{c \cdot} - \W\pow{k}_{c \cdot})
    \end{align*}
    is equal to $\grad_{\W_{c \cdot}} \Psi(\U\pow{k}, \btheta\pow{k}, \W\pow{k})$. To simplify the expression for the unsupervised loss, note that because $\W\pow{k}_{c \cdot} = \W\pow{k, c}_{c \cdot}$, it holds that
    \begin{align*}
        \grad_{\W_{c \cdot}} \Lunsup(\W\pow{k}, \U\pow{k}) - \grad_{\W_{c \cdot}} \Lunsup(\W\pow{k, c}, \U\pow{k}) = \grad_{\W_{c \cdot}} L(\W\pow{k}) - \grad_{\W_{c \cdot}} L(\W\pow{k, c}).
    \end{align*}
    For the supervised loss, defining the expression below as zero for $c > M$, because $\W\pow{k, c-1}_{c \cdot} = \W\pow{k-1}_{c \cdot}$, we have that
    \begin{align*}
        &\grad_{\W_{c \cdot}} \Lsup(\W\pow{k}, \btheta\pow{k}) - \grad_{\W_{c \cdot}} \Lsup(\W\pow{k, c-1}, \btheta\pow{k}) \\
        &= \frac{\lambda}{N} \sum_{i=1}^N \grad_{\W_{c \cdot}}  \sbr{\ell_c((\W_{c \cdot}\pow{k})^\top \z_i, y_{i, c}, \btheta_c\pow{k}) - \ell_c((\W_{c \cdot}\pow{k-1})^\top \z_i, y_{i, c}, \btheta_c\pow{k})}.
    \end{align*}
    By combining all three steps with the subdifferential calculus rule
    \begin{align*}
        \partial \Psi(\U\pow{k}, \btheta\pow{k}, \W\pow{k}) = \partial_{\U} \Psi(\U\pow{k}, \btheta\pow{k}, \W\pow{k}) \times \partial_{\btheta} \Psi(\U\pow{k}, \btheta\pow{k}, \W\pow{k})  \times \partial_{\W} \Psi(\U\pow{k}, \btheta\pow{k}, \W\pow{k}),
    \end{align*}
    we complete the proof.
\end{proof}

To complete the analysis, we upper bound the norm of the subgradient
\begin{align}
    \norm{\p{G\pow{k}_{\U}, G\pow{k}_{\btheta}, G\pow{k}_{\W}}}_{\Fro} \leq \norm{G\pow{k}_{\U}}_{\Fro} + \norm{G\pow{k}_{\btheta}}_2 + \norm{G\pow{k}_{\W}}_{\Fro}.\label{eq:subdiff_bound_triangle}
\end{align}
This is where the key assumptions are used.
\begin{proposition}\label{prop:subdiff_bd2}
    Assume the conditions of \Cref{lem:monotonicity} and additionally, \Cref{asm:bdd1,asm:bdd2,asm:bdd3}. Let $L = (L_{\U}, L_{\btheta}, L_{\W})$, $\eta = ( \etaa, \etap, \etau)$, and $\Z = (\z_1, \ldots, \z_N)$.
    Then, there exists a constant 
    \begin{align*}
        C_0 \equiv C_0(L, \eta, \Z, \mu, \sigma)
    \end{align*}
    such that for $k = 1, \ldots, K$, it holds that
    \begin{align*}
        \norm{\p{G\pow{k}_{\U}, G\pow{k}_{\btheta}, G\pow{k}_{\W}}}_{\Fro}^2 \leq C_0 I\pow{k}.
    \end{align*}
\end{proposition}
\begin{proof}
    Starting from~\eqref{eq:subdiff_bound_triangle}, we bound each norm using the formulas from \Cref{lem:subdiff_bd1}. In the following, we use $C_i \equiv C_i(L, \eta, \Z, \mu, \sigma) \geq 0$ as a constant. Recall the function $f$ from~\eqref{eq:supergaussian}.
    Under \Cref{asm:bdd1} and \Cref{asm:bdd3}, it holds that
    \begin{align*}
        \norm{G\pow{k}_{\U}}_{\Fro} &\leq \frac{1}{\etaa} \norm{\U\pow{k-1} - \U\pow{k}}_{\Fro} + \norm{\grad_{\U} F_\mu(\W\pow{k}, \btheta\pow{k}, \U\pow{k}) - \grad_{\U} F_\mu(\W\pow{k-1}, \btheta\pow{k-1}, \U\pow{k})}_{\Fro}\\
        &\leq  \frac{1}{\etaa} \norm{\U\pow{k-1} - \U\pow{k}}_{\Fro} + \frac{1}{N}\p{\sum_{i, c, t}\p{ f'(U\pow{k}_{i, c, t}) - f'(U\pow{k-1}_{i, c, t})}^2}^{1/2}\\
        &\quad + \frac{1}{2N}\p{\sum_{i, c, t}\p{  \ipsmall{\W\pow{k}_{c, \cdot} \z_{i, t}}^2 -\ipsmall{\W\pow{k-1}_{c, \cdot} \z_{i, t}}^2}^2}^{1/2}\\
        &\leq \p{\frac{1}{\etaa} + \frac{L_{\U}}{N}} \norm{\U\pow{k-1} - \U\pow{k}}_{\Fro} + C_1 \p{\sum_{c=1}^C \norm{\W\pow{k}_{c, \cdot} - \W\pow{k-1}_{c, \cdot}}_2^2}^{1/2}\\
        &= \p{\frac{1}{\etaa} + \frac{L_{\U}}{N}} \norm{\U\pow{k-1} - \U\pow{k}}_{\Fro} + C_1 \norm{\W\pow{k} - \W\pow{k-1}}_{\Fro}.
    \end{align*}
    Next, under \Cref{asm:bdd1} and the smoothness conditions of \Cref{lem:monotonicity}, we have 
    \begin{align*}
        \norm{G\pow{k}_{\btheta}}_2 &\leq \frac{1}{\etap}\norm{\btheta\pow{k} - \btheta\pow{k-1}}_2 + \norm{\grad_{\btheta} F_\mu(\W\pow{k}, \btheta\pow{k}, \U\pow{k}) - \grad_{\btheta} F_\mu(\W\pow{k-1}, \btheta\pow{k-1}, \U\pow{k})}_2\\
        &\leq \frac{1}{\etap}\norm{\btheta\pow{k} - \btheta\pow{k-1}}_2 + \norm{\grad_{\btheta} F_\mu(\W\pow{k}, \btheta\pow{k}, \U\pow{k}) - \grad_{\btheta} F_\mu(\W\pow{k}, \btheta\pow{k-1}, \U\pow{k})}_2\\
        &\quad + \norm{\grad_{\btheta} F_\mu(\W\pow{k}, \btheta\pow{k-1}, \U\pow{k}) - \grad_{\btheta} F_\mu(\W\pow{k-1}, \btheta\pow{k-1}, \U\pow{k})}_2\\
        &\leq \p{\frac{1}{\etap} + C_2}\norm{\btheta\pow{k} - \btheta\pow{k-1}}_2 + C_3 \norm{\W\pow{k} - \W\pow{k-1}}_{\Fro},
    \end{align*}
    where in the bound on the third term, we use the mixed smoothness constant $\bar{L}$ from \Cref{asm:smoothness}.
    Finally, under \Cref{asm:bdd1} and \Cref{asm:bdd2}, we have
    \begin{align*}
        \norm{G\pow{k}_{\W}}_{\Fro} &\leq \p{\frac{1}{2\etau}+C_4}\norm{\W\pow{k} - \W\pow{k-1}}_{\Fro} + \p{\sum_{c=1}^C \norm{\grad_{\W_{c \cdot}} L(\W\pow{k}) - \grad_{\W_{c \cdot}} L(\W\pow{k, c})}_2^2 }^{1/2}\\
        &\leq \p{\frac{1}{2\etau}+C_4}\norm{\W\pow{k} - \W\pow{k-1}}_{\Fro} + \p{\sum_{c=1}^C \sum_{j=c+1}^C \norm{\grad_{\W_{c \cdot}} L(\W\pow{k, j}) - \grad_{\W_{c \cdot}} L(\W\pow{k, j-1})}_2^2 }^{1/2}\\
        &\leq \p{\frac{1}{2\etau}+C_4}\norm{\W\pow{k} - \W\pow{k-1}}_{\Fro} + \frac{1}{\sigma} \p{\sum_{c=1}^C \sum_{j=c+1}^C \norm{\W\pow{k, j} - \W\pow{k, j-1}}_{\Fro}^2 }^{1/2}\\
        &= \p{\frac{1}{2\etau}+C_4}\norm{\W\pow{k} - \W\pow{k-1}}_{\Fro} + \frac{1}{\sigma} \p{\sum_{c=1}^C \sum_{j=c+1}^C \norm{\W\pow{k, j}_{j \cdot} - \W\pow{k, j-1}_{j\cdot}}_{2}^2 }^{1/2}\\
        &\leq \p{\frac{1}{2\etau}+C_4 + \frac{C}{\sigma}}\norm{\W\pow{k} - \W\pow{k-1}}_{\Fro}.
    \end{align*}
    Combine the three bounds and the inequality $(a + b + c)^2 \leq 3 (a^2 + b^2 + c^2)$ to achieve the desired result.
\end{proof}

\end{document}